\documentclass[nohyperref]{article}

\usepackage{microtype}
\usepackage{graphicx}
\usepackage{subfigure}
\usepackage{booktabs} % for professional tables

\usepackage{amssymb}
\usepackage{tabularx}
\usepackage{wrapfig}
\usepackage{lipsum}
\usepackage{caption}
\usepackage{thmtools,thm-restate}
\usepackage{nicefrac}
\usepackage{hyperref}

\usepackage[accepted]{icml2023}

\usepackage{amsmath}
\usepackage{amssymb}
\usepackage{mathtools}
\usepackage{amsthm}

\usepackage[capitalize,noabbrev]{cleveref}

\usepackage{xcolor}

\newcommand{\norm}[1]{\left\Vert#1\right\Vert}

\newcommand{\abs}[1]{\left\vert#1\right\vert}

\newcommand{\set}[1]{\left\{#1\right\}}
\newcommand{\parr}[1]{\left (#1\right )}
\newcommand{\brac}[1]{\left [#1\right ]}

\newcommand{\Real}{\mathbb R}

\newcommand{\eps}{\varepsilon}
\newcommand{\too}{\rightarrow}

\definecolor{mygray}{gray}{0.95}

\newcommand{\eg}{{e.g.}}
\newcommand{\ie}{{i.e.}}

\usepackage{amsmath,amsfonts,bm}

\def\eqref#1{equation~\ref{#1}}
\def\1{\bm{1}}

\def\eps{{\epsilon}}

\DeclareMathAlphabet{\mathsfit}{\encodingdefault}{\sfdefault}{m}{sl}
\SetMathAlphabet{\mathsfit}{bold}{\encodingdefault}{\sfdefault}{bx}{n}

\def\gA{{\mathcal{A}}}
\def\gB{{\mathcal{B}}}

\def\gE{{\mathcal{E}}}

\def\gL{{\mathcal{L}}}

\def\gN{{\mathcal{N}}}

\def\gP{{\mathcal{P}}}

\def\gT{{\mathcal{T}}}

\newcommand{\E}{\mathbb{E}}

\newcommand{\R}{\mathbb{R}}

\usepackage{tabularx,ragged2e,booktabs,caption}
\newcolumntype{C}[1]{>{\Centering}m{#1}}

\newcolumntype{Z}[1]{>{\Left}m{#1}}

\newcommand*\rot[1]{\rotatebox{90}{#1}}

\theoremstyle{plain}
\newtheorem{theorem}{Theorem}[section]

\newtheorem{lemma}[theorem]{Lemma}
\newtheorem{corollary}[theorem]{Corollary}
\theoremstyle{definition}

\theoremstyle{remark}

\makeatletter
\newtheorem*{rep@theorem}{\rep@title}
\newcommand{\newreptheorem}[2]{%
\newenvironment{rep#1}[1]{%
 \def\rep@title{\textbf{#2} \ref{##1}}%
 \begin{rep@theorem}}%
 {\end{rep@theorem}}}
\makeatother

\newreptheorem{theorem}{Theorem}
\newreptheorem{proposition}{Proposition}
\newreptheorem{lemma}{Lemma}
\newreptheorem{corollary}{Corollary}
\usepackage[textsize=tiny]{todonotes}

\icmltitlerunning{On Kinetic Optimal Probability Paths for Generative Models}

\begin{document}

\twocolumn[
\icmltitle{On Kinetic Optimal Probability Paths for Generative Models}

% It is OKAY to include author information, even for blind
% submissions: the style file will automatically remove it for you
% unless you've provided the [accepted] option to the icml2023
% package.

% List of affiliations: The first argument should be a (short)
% identifier you will use later to specify author affiliations
% Academic affiliations should list Department, University, City, Region, Country
% Industry affiliations should list Company, City, Region, Country

% You can specify symbols, otherwise they are numbered in order.
% Ideally, you should not use this facility. Affiliations will be numbered
% in order of appearance and this is the preferred way.
\icmlsetsymbol{equal}{*}

\begin{icmlauthorlist}
\icmlauthor{Neta Shaul}{yyy}
\icmlauthor{Ricky T. Q. Chen}{comp}
\icmlauthor{Maximilian Nickel}{comp}
\icmlauthor{Matt Le}{comp}
\icmlauthor{Yaron Lipman}{comp,yyy}
\end{icmlauthorlist}

%\icmlaffiliation{yyy}{Department of XXX, University of YYY, Location, Country}
\icmlaffiliation{yyy}{Weizmann Institute of Science}
%\icmlaffiliation{comp}{Company Name, Location, Country}
\icmlaffiliation{comp}{Meta AI (FAIR)}
%\icmlaffiliation{sch}{School of ZZZ, Institute of WWW, Location, Country}

\icmlcorrespondingauthor{Neta Shaul}{Neta.Shaul@weizmann.ac.il}
%\icmlcorrespondingauthor{Firstname2 Lastname2}{first2.last2@www.uk}

% You may provide any keywords that you
% find helpful for describing your paper; these are used to populate
% the "keywords" metadata in the PDF but will not be shown in the document
\icmlkeywords{Machine Learning, ICML}

\vskip 0.3in
]

% this must go after the closing bracket ] following \twocolumn[ ...

% This command actually creates the footnote in the first column
% listing the affiliations and the copyright notice.
% The command takes one argument, which is text to display at the start of the footnote.
% The \icmlEqualContribution command is standard text for equal contribution.
% Remove it (just {}) if you do not need this facility.

\printAffiliationsAndNotice{}  % leave blank if no need to mention equal contribution
%\printAffiliationsAndNotice{\icmlEqualContribution} % otherwise use the standard text.

\begin{abstract}
Recent successful generative models are trained by fitting a neural network to an a-priori defined tractable probability density path taking noise to training examples. In this paper we investigate the space of Gaussian probability paths, which includes diffusion paths as an instance, and look for an optimal member in some useful sense. In particular, minimizing the Kinetic Energy (KE) of a path is known to make particles' trajectories simple, hence easier to sample, and empirically improve performance in terms of likelihood of unseen data and sample generation quality. We investigate Kinetic Optimal (KO) Gaussian paths and offer the following observations: (i) We show the KE takes a simplified form on the space of Gaussian paths, where the data is incorporated only through a single, one dimensional scalar function, called the \emph{data separation function}. (ii) We characterize the KO solutions with a one dimensional ODE. (iii) We approximate data-dependent KO paths by approximating the data separation function and minimizing the KE. (iv) We prove that the data separation function converges to $1$ in the general case of arbitrary normalized dataset consisting of $n$ samples in $d$ dimension as $n/\sqrt{d}\too 0$. A consequence of this result is that the Conditional Optimal Transport (Cond-OT) path becomes \emph{kinetic optimal} as $n/\sqrt{d}\too 0$. We further support this theory with empirical experiments on ImageNet.
\end{abstract}

\section{Introduction}

In recent years, deep generative models have become very powerful both in terms of sample quality and density estimation \cite{Rombach2021stablediffusion,gafni2022make}. The recent revolution is mainly led by a class of time-dependent generative models which make use of predefined probability paths $p_t$, constituting a process interpolating between the target (data) distribution $q$ and the prior (noise) distribution $p$. The training procedure of these models can be generally described as a regression task of a neural network, $v_{t}(x;\theta)$, to some vector quantity, $u_t$, that allows sampling from the interpolation process:
\begin{equation}\label{e:loss}
    \gL(\theta) = \E_{p_t(x)} \ell(v_{t}(x;\theta), u_t(x)),
\end{equation}
where $\ell(v,u)$ is some cost function, and $\gL$ is the total loss objective that is stochastically estimated and optimized during training. A notable family of algorithms that fits into this framework are diffusion models \cite{ho2020denoising,song2020score} as well as the recent Flow-Matching models, \cite{lipman2022flow,albergo2022si,liu2022flow,neklyudov2022action} generalizing the principles used in diffusion models training and apply them to simulation-free continuous normalizing flows (CNFs, \cite{chen2018neural}) training. In particular, \cite{song2020score,lipman2022flow} show that diffusion models also fall into the simulation-free CNF framework. CNFs parameterize the generation process via a flow vector field $u_t$, which defines the probability path $p_t$. 

In order to make the loss in \eqref{e:loss} tractable, all these methods are confined to a family of \emph{Gaussian probability paths}
\begin{equation}\label{e:P}
 \gP=\set{p_t \ \Bigg \vert \  p_t(x) = \int p_t(x|x_1)q(x_1)dx_1}   
\end{equation}
defined by the conditional Gaussian probabilities $p_t(x|x_1) = \gN(x|a_t x_1, m_t^2 I)$, on specific training examples $x_1\sim q$. The space $\gP$ allows scalable unbiased approximations of the loss in \eqref{e:loss} and therefore it is of particular interest. Within this space, diffusion models construct $p_t$ as a diffusion markov process and \cite{lipman2022flow,liu2022flow} draw inspiration from optimal transport to define the conditional probabilities. The different choices of previously explored paths are based on known processes, empirical heuristics \cite{karras2022elucidating} or learned paths \cite{nichol2021improved,kingma2021vdm}, but there is no theoretical result analyzing the \emph{optimality} of those paths. 

\pagebreak
The goal of this paper is to investigate the space of probability paths, $\gP$, and single out an \emph{optimal one} in some well defined useful sense. We consider the Kinetic Energy (KE) as the measure of path optimality. The Kinetic energy is directly tied to Optimal Transport solutions in the dynamic formulation \cite{benamou2000computational,villani2009optimal} where the optimal probability path $p_t$ minimizes the KE. Therefore, reducing the KE simplifies the trajectories of individual particles, allowing faster sampling and empirically improved performance. Our key observation is that the KE over $\gP$ takes a simplified form, where its dependence on the data distribution $q$ is summarized in a single, one dimensional, scalar function $\lambda:[0,\infty)\too[0,1]$. Intuitively, $\lambda$ measures the separation of the data samples at different scales, and therefore we call it the \emph{data separation function}. We then characterize the minimizers of KE over $\gP$. These observations allow us to approximate Kinetic Optimal (KO) probability paths for different datasets $q$. Note that $\gP$ is optimized over the Gaussian paths and therefore the KO path is usually not the Optimal Transport path. 

Surprisingly, for high dimensional data we are able to characterize Kinetic Optimal solutions. We prove that for arbitrary normalized datasets consisting of $n$ data points in $\Real^d$ the probability path defined by the Conditional Optimal-Transport (Cond-OT) path suggested by \cite{lipman2022flow} becomes kinetic optimal as 
\begin{equation}
    \frac{n}{\sqrt{d}} \too 0.
\end{equation}
Empirically, we show evidence that in practice Cond-OT is KO even for lower dimensions than predicted by the above asymptotics, and show that KO paths provides improvement in the KE of trained models over Cond-OT paths in low dimensions and that this improvement is diminished as dimension increases, as predicted by the theory. 

\section{Preliminaries}
We consider $\Real^d$ as our data domain with data points $x\in\Real^d$.  Probability densities over $\Real^d$ will be denoted by $p,q$. With a slight abuse of notations, but as is usually done, random variables will also be denoted as $x,x_0,x_1$, where a random variable $x$ distributed according to $p$ is denoted $x\sim p(x)$. A time-dependent vector field (VF) is a smooth function $u:[0,1]\times\Real^d\too\Real^d$. A VF defines a \emph{flow}, which is a diffeomorphism $\Phi:\Real^d\too\Real^d$, defined via a solution  $\phi_t(x)$ to the ordinary differential equation (ODE):
\begin{align}
    \frac{d}{dt}\phi_t(x) &= u_t(\phi_t(x))\\
    \phi_0(x) &= x
\end{align}
and setting $\Phi(x)=\phi_1(x)$. 
A probability density path is a time dependent probability density $p_t$, $t\in[0,1]$, smooth in $t$. Given a probability density path $p_t$, it is said to be \emph{generated} by $u_t$ from $p$ if
\begin{equation}\label{e:generating}
    p_t = (\phi_{t})_{\#} p
\end{equation}
where the $\#$ is the push-forward operation of probability densities. 
Continuous Normalizing Flows (CNFs) \cite{chen2018neural} is a general methodology for learning generative models $\Phi:\Real^d\too\Real^d$ by modeling $u_t$ as a neural network. All the generative models considered in this paper can be seen as instances of CNFs, trained with different losses and probability path supervision. 

\section{Optimal probability paths} 
Recent generative models learn continuous normalizing flows by regressing a vector field defining some fixed probability density path $p_t$ that is known to take a simple prior normal distribution (noise) $p=\gN(0,I)$ to a more complex distribution $q$ (data); \eqref{e:loss} describes a general loss objective for these models. In practice, the distribution $q$ is constructed (\eg, as a sum of delta distributions or a Gaussian mixture model) from a training set, which is sampled from some unknown distribution. Without limiting the following discussion, we will assume the target distribution $q$ is normalized, \ie, 
\begin{align} \label{e:q_mean}
    \E_{q(x_1)} x_1 &=0 \\ \label{e:q_var}
    %\frac{1}{d}\E \norm{x_1}^2 &= \sigma^2
    \frac{1}{d}\E_{q(x_1)}\norm{x_1}^2 &= 1
\end{align}
The second condition can be understood as setting the average variance (across coordinates of data samples $x_1\in \Real^d$) to be $1$. 

\paragraph{Probability paths with affine conditionals.}
The probability density path $p_t$ is defined by marginalizing conditional probabilities $p_t(x|x_1)$, 
\begin{equation}\label{e:p_t}
    p_t(x) = \int p_t(x|x_1)q(x_1)dx_1.
\end{equation}
Our goal in this paper is to analyze a popular class of such probability paths (or paths, in short), referred here as \emph{Gaussian paths}, defined by the \emph{affine conditional probabilities} 
\begin{equation}\label{eq:p_cond}
    p_t(x|x_1) = \gN(x|a_t x_1, m_t^2 I),
\end{equation}
where $(a_t,m_t)\in \gA$, and
\begin{equation}\label{e:A}
    \gA = \set{(a_t,m_t) \ \Bigg \vert \substack{a_t,m_t:[0,1]\too[0,\infty] \\ a_0=0=m_1 \\ a_1=1=m_0}}, %, \text{smooth in }t}
\end{equation}
where we assume all functions in $\gA$ are continuous in $[0,1]$ and smooth in $(0,1)$.

Note that any choice of $(a_t,m_t)\in \gA$ will make $p_0=p$ and $p_1=q$ (in the limit $t\too 1$)\footnote{Note that we use a forward time convention, where $t=0$ corresponds to noise and $t=1$ to data, differently from the standard convention in Diffusion models that use reverse timing. }. The reason we call such conditionals affine is that they can be seen as affine maps of a standard Gaussian: if $x_0\sim p=\gN(0,I)$ then $x = a_t x_1 + m_t x_0 \sim p_t(x|x_1)$.
Affine conditionals have been used in most previous methods, as described next. 

\emph{Diffusion models} use affine conditionals of the form \citep{ho2020denoising}:
\begin{equation}
    m_t =\sqrt{1-\xi_{1-t}^2} , \quad a_t= \xi_{1-t},
\end{equation}
where $\xi_{t}= e^{-\frac{1}{2}\int_0^t \beta(s)ds}$, and $\beta$ is the noise scale function.
\emph{Stochastic Interpolants} \citep{albergo2022si} are suggesting conditionals that are parameterized with trigonometric coefficients, similar also to the so-called cosine scheduling used in Diffusion \citep{salimans2022progressive}
\begin{equation}\label{e:SI}
    m_t = \cos{\frac{\pi}{2}t}, \quad a_t= \sin{\frac{\pi}{2}t}. 
\end{equation}

\emph{Flow Matching} \cite{lipman2022flow,liu2022flow} suggests linear coefficients (in $t$) called Conditional Optimal Transport (Cond-OT), 
\begin{equation}\label{eq:ot_cond}
    m_t = 1-t, \quad a_t = t.
\end{equation}

The different choices of affine conditionals in the probability path were justified mostly heuristically without any principled method for studying \emph{optimality} under some sensible criterion. In this work we will study optimality of affine conditionals from the point of view of \emph{Kinetic Energy}. To define the Kinetic Energy we must attach to the probability path a generating vector field (speed) in the sense of \eqref{e:generating}. 

\paragraph{Generating vector field of paths.} A natural choice for defining a generating vector field $u_t(x)$ for the path $p_t$ in the sense of \eqref{e:generating} works as follows (see \citet{lipman2022flow} for proofs of the following facts): first, the (conditional) vector field $u_t(x|x_1)$ that generates the affine conditional $p_t(x|x_1)$ from $p(x)$ is 
\begin{align}\label{e:u_t_cond}
    u_t(x|x_1) &= \alpha_t x + \beta_t x_1
\end{align}
where
\begin{equation}\label{e:def_beta_alpha}
    \alpha_t = \frac{\dot{m}_t}{m_t}, \quad \beta_t = \dot{a}_t-a_t \frac{\dot{m}_t}{m_t}
\end{equation}
where $\dot{a}_t=\frac{d}{dt}a_t$ denotes the time derivative. Further note that $u_t(x|x_1)$ is an affine function in variable $x$, and therefore a gradient in $x$ of some potential making this choice of conditional vector field $u_t(x|x_1)$ unique \cite{neklyudov2022action}. Secondly, the \emph{marginal} vector field $u_t(x)$ that generates $p_t(x)$ from $p(x)$ is provided via the following formula that aggregates these conditional vector fields
\begin{equation}\label{e:u_t}
    u_t(x) = \int u_t(x|x_1) \frac{p_t(x|x_1) q(x_1)}{p_t(x)}dx_1.
\end{equation}
In the next section we define the Kinetic Energy that uses both the marginal probability path $p_t(x)$ and its generating VF $u_t(x)$.

\subsection{Kinetic energy of affine paths}
Our goal is to study optimal choice of affine conditionals that lead to probability paths with minimal \emph{Kinetic Energy}. 
The Kinetic Energy (KE) is defined for a pair $(p_t,u_t)$ of a probability path and its generating VF (in the sense of \eqref{e:generating}):
\begin{equation}\label{e:ke}
    \gE(a_t,m_t) = \frac{1}{d}\E_{t,p_t(x)} \norm{u_t(x)}^2,
\end{equation}
where $p_t,u_t$ are defined in equations \ref{e:p_t} and \ref{e:u_t}, respectively. The Kinetic Energy measures the \emph{action} of the path, by aggregating the Kinetic Energy of individual particles' paths. The main motivation for studying the Kinetic Energy is that reducing the Kinetic Energy simplifies the overall particles' path. Namely, favors paths that are straight lines with constant speed parameterization. Such paths are in fact globally minimizing the Kinetic Energy and in turn lead to Optimal Transport (OT) interpolants, where the minimal KE value realizes the Wasserstein distance of the source and target probabilities, $\gT_2(p,q)$  \cite{mccann1997convexity,villani2021topics}. 
In the context of generative models, simplifying the paths improves sampling efficiency of the model (ideally, OT can be sampled with a single function evaluation), and is empirically shown to improve performance and training speed \cite{lipman2022flow,albergo2022si,liu2022flow}. Note however, that OT paths are generally \emph{outside} the class of probability paths with affine conditionals. Nevertheless, since affine conditionals are of particular interest due to the fact that they are amenable to large scale training, we investigate the kinetic optimal path within this class of paths.

A proxy energy that will be used to investigate the Kinetic Energy is the \emph{Conditional Kinetic Energy} (CKE), that is simply the aggregate kinetic energies of the conditional vector fields,
\begin{equation}\label{e:ke_cond}
    \gE_c(a_t,m_t) = \frac{1}{d} \E_{t,q(x_1),p_t(x|x_1)}\norm{u_t(x|x_1)}^2
\end{equation}
The CKE, for affine conditionals can be expressed as 
\begin{equation}\label{e:ke_cond_closed_form}
    \gE_c(a_t,m_t) = \int_0^1 \parr{\dot{m}_t^2 + \dot{a}_t^2 }dt.
\end{equation}
Next, the KE takes the form 
\begin{equation}\label{e:ke_explicit}
    \hspace{-2pt}\gE(a_t,m_t) = \gE_c(a_t,m_t) - \int_0^1  {\beta_t^2 }\hspace{-2pt} \parr{\hspace{-2pt}1\hspace{-2pt}-\hspace{-2pt}\lambda\hspace{-2pt}\parr{\frac{a_t}{m_t}}\hspace{-2pt}} dt,
\end{equation}
where $\lambda=\lambda_q:[0,\infty]\too [0,1]$ is a univariate scalar function defined solely in terms of the target distribution $q$, and bounded by $1$, \ie, $\lambda(s)\leq 1$. Computations supporting these derivations can be found in \Cref{aa:ke_comp}. The $\lambda$ function measures the separation in the data distribution $q$ and therefore we call it the \emph{data separation function}. Before providing an explicit expression for the data separation function, we note that having access to this function, one can optimize \eqref{e:ke_explicit} to find optimal conditional marginals $(a_t,m_t)\in \gA$. In the next section we discuss the data separation function $\lambda=\lambda_q$, followed by characterizing optimal solutions of $\gE_c$. Lastly, note that $\lambda(s)\leq 1$ implies that $\gE_c\geq \gE$.

\paragraph{The data separation function.}
The data separation function $\lambda=\lambda_q$ summarizes the contribution of the data $q$ to the Kinetic Energy (\eqref{e:ke_explicit}) in a single, univariate, scalar $[0,\infty]\too[0,1]$ function. It therefore provides a significant reduction in complexity that is due to the affine conditional probabilities. 

To write $\lambda=\lambda_q$ explicitly, we first denote by 
\begin{equation}
    \rho_s(x) = \int \rho_s(x|x_1)q(x_1)dx_1
\end{equation}
the result of adding Gaussian noise, $\rho_s(x|\mu) = \gN(x|\mu, s^{-2}I)$, of scale $s^{-1}$, to the data $q$. Now with Bayes' rule we compute the probability of data point $x_1$ given a noisy data point $x$,
$$\rho_s(x_1|x)=\frac{\rho_s(x|x_1)q(x_1)}{\rho_s(x)}.$$
Then the data separation function is
\begin{equation}\label{eq:lambda_s}
    \lambda(s) = \frac{1}{d}\E_{\rho_s(x)}\norm{\E_{\rho_s(x_1|x)}x_1}^2.
\end{equation}
See \Cref{aa:ke_comp} for detailed derivation of the $\lambda$ function. 

To understand the data separation function intuitively, consider discrete data $q$ which consists of data points $x_i$, $i\in [n]$: if this data is well separated at scale $s^{-1}$, then a noisy sample $x=x_i + \eps \sim \rho_s(x)$ will be closer to the data point $x_i$ than any other data point $x_j$. Hence, the softmax term $\rho_s(x_1|x)\approx \delta_{x_i}(x)$, where $\delta_{x_i}(x)$ is the delta distribution centered at $x_i$. 
Therefore in this case
$$\lambda(s) \approx \frac{1}{d} \E_{\rho_s(x|x_1),q(x_1)}\norm{x_1}^2 = \frac{1}{d} \E_{q(x_1)}\norm{x_1}^2 = 1,$$
where in the last equality we used our data assumption in \eqref{e:q_var}. So, the closer to $1$ the function value $\lambda(s)$ is, $s\in[0,\infty]$, the more separated $q$ at scale $s^{-1}$. Note however that for data with average variance of $1$ (see \eqref{e:q_var} again), $\lambda$ is globally bounded by $1$. Indeed using Jensen's inequality we have
\begin{equation}\label{e:lambda_bound}
    \lambda(s) \leq \frac{1}{d}\E_{\substack{\rho_s(x)\\\rho_s(x_1|x)}}\norm{x_1}^2 = \frac{1}{d}\E_{q(x_1)}\norm{x_1}^2 = 1.
\end{equation}

\subsection{Kinetic optimal solutions}
In this section we characterize kinetic optimal paths within the space of Gaussian paths. These solution will be depending on the data separation function $\lambda$.

To facilitate the variational analysis of the KE functional in \eqref{e:ke_explicit} we perform a change of variables and represent the conditional probability parameters $(a_t,m_t)\in \gA$ via 
\begin{equation}\label{e:polar_change}
    \begin{bmatrix}
    a_t \\ m_t
    \end{bmatrix} = 
    r_t \begin{bmatrix}
    \sin(\theta_t) \\ \cos(\theta_t)
    \end{bmatrix},
\end{equation}
where now our optimization space of affine conditionals becomes 
\begin{equation}\label{e:B}
    \gB = \set{(r_t,\theta_t) \ \Bigg \vert \substack{r_t:[0,1]\too[0,\infty)\\ \theta_t:[0,1]\too [0,\frac{\pi}{2}] \\ r_0=1=r_1 \\ \theta_0=0, \theta_1=\frac{\pi}{2}}}. %, \text{smooth in }t}
\end{equation}

Plugging these coordinates in \eqref{e:ke_explicit} leads to the following form of the KE:
\begin{equation}\label{eq:ke_polar_form}
    \gE(r_t,\theta_t) = \int_0^1 \dot{r}_t^2 + r_t^2 \dot{\theta}_t^2 \overset{\gamma(\theta_t)}{\overbrace{ \brac{1-\frac{1-\lambda\parr{\tan(\theta_t)}}{\cos^2(\theta_t)}} }}  dt.
\end{equation}
The stationary solutions of \eqref{eq:ke_polar_form} obey the Euler-Lagrange equations, which is a set of two ordinary differential equations (ODEs) with two unknown functions $r_t,\theta_t$. Surprisingly, it turns out that under the mild extra condition, \ie, $\gamma(\theta_t)>0$, this system of equations is separable and reduces to two rather simple ODEs, with $r_t$ solvable analytically up to a single unknown parameter, and $\theta_t$ is characterized with a first order ODE: 

\begin{theorem}\label{thm:main_optimal_path}
The minimizer of the Kinetic Energy (\eqref{eq:ke_polar_form}) over all conditional probabilities $(r_t,\theta_t)\in \gB$ (\eqref{e:B}) such that $\gamma(\theta_t)>0$ for $t>0$ satisfy 
\begin{align}
    &r_t = \sqrt{1 -bt  + b t^2},\label{eq:op_r}\\
    &\dot{\theta}_t = \frac{1}{1 -bt + bt^2} \sqrt{\frac{b-\frac{1}{4}b^2}{\gamma(\theta_t)}}\label{eq:op_theta}
\end{align}
where $b\in [0,4]$ is determined by the boundary condition on $\theta_t$.
\end{theorem}
The proof of this theorem is provided in \Cref{a:optimal_solutions}. To better grasp the meaning of the condition $\gamma(\theta_t)>0$ we show in \Cref{aa:gamma_cond} that it is equivalent to 
\begin{equation}\label{e:gamma_cond}
    \lambda(s) \geq \frac{s^2}{1+s^2}, \quad s\in [0,\infty].
\end{equation}
This implies a condition on the target data distribution $q$ for which our optimality result holds. To get a better sense of this condition in \Cref{aa:gamma_cond} we also prove that if $q(x)=\gN(0,I)$, namely a standard Gaussian, then $\lambda(s)=\frac{s^2}{1+s^2}$. Hence, the condition in \eqref{e:gamma_cond} requires the separation of $q$ to be at-least that of a standard Gaussian. All the empirical data $q$ we tested in this paper has at-least the separation of a standard Gaussian, namely satisfy \eqref{e:gamma_cond}.

\paragraph{Conditional kinetic optimal solutions.}
One case where the optimal solution is known analytically is when $\gamma(\theta_t)\equiv 1$, which happens when $\lambda(s)\equiv 1$. In this case, \eqref{e:ke_explicit} shows that the Kinetic Energy and the Conditional Kinetic Energy are equal, that is $\gE=\gE_c$. Consequently solving the Euler-Lagrange equations can be done directly for \eqref{e:ke_cond_closed_form} and the solutions are the Cond-OT conditionals in \eqref{eq:ot_cond}. Alternatively, Theorem \ref{thm:main_optimal_path} can be used and \eqref{eq:op_theta} can be solved to get $\theta_t$ as well as the parameter $b$ leading to the solution
\begin{equation*}
    r_t = \sqrt{(1-t)^2 +t^2} \quad \theta_t=\arctan\parr{\frac{t}{1-t}}
\end{equation*}
and using \eqref{e:polar_change} to solve for $a_t,m_t$ we get again the Cond-OT conditionals in \eqref{eq:ot_cond}. 

\paragraph{Non-analytic solutions.}
When an analytic solution of \eqref{eq:op_theta} cannot be achieved one can approximate the data separation function $\lambda$ numerically, and then optimize the Kinetic Energy $\gE$ directly for $r_t,\theta_t$, where for $r_t$ it is convenient to use the one parameter solution family in \eqref{eq:op_r}. We come back to this procedure in the experiments section.

% Figure of lambda 
\begin{figure*}[h!]
    \centering
    \begin{tabular}{@{\hspace{0pt}}c@{\hspace{0.03\textwidth}}c@{\hspace{0.03\textwidth}}c@{\hspace{0pt}}}
         \includegraphics[width=0.30\textwidth]{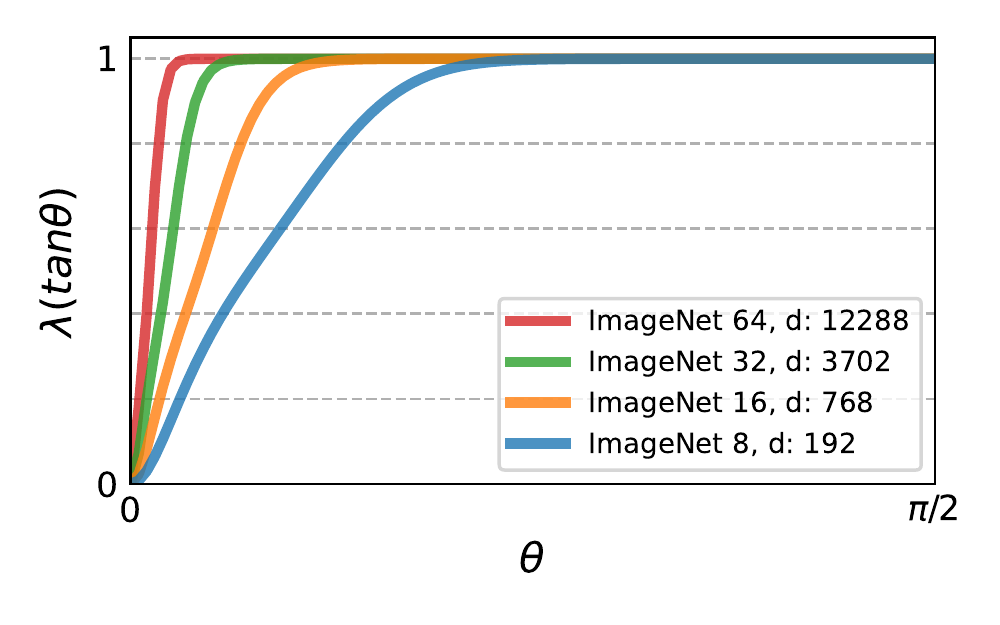} & \includegraphics[width=0.30\textwidth]{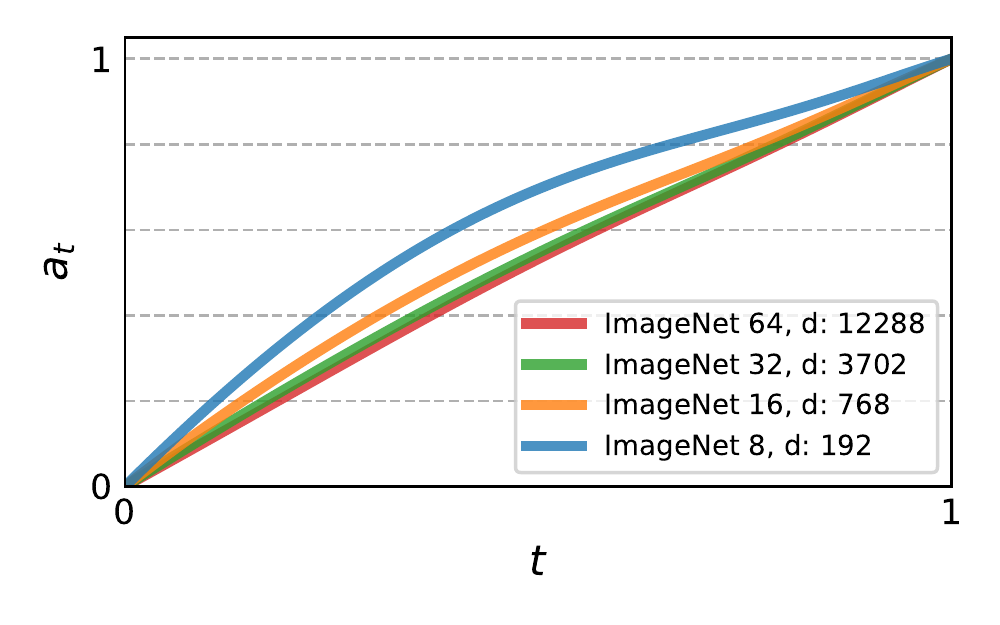} &
         \includegraphics[width=0.30\textwidth]{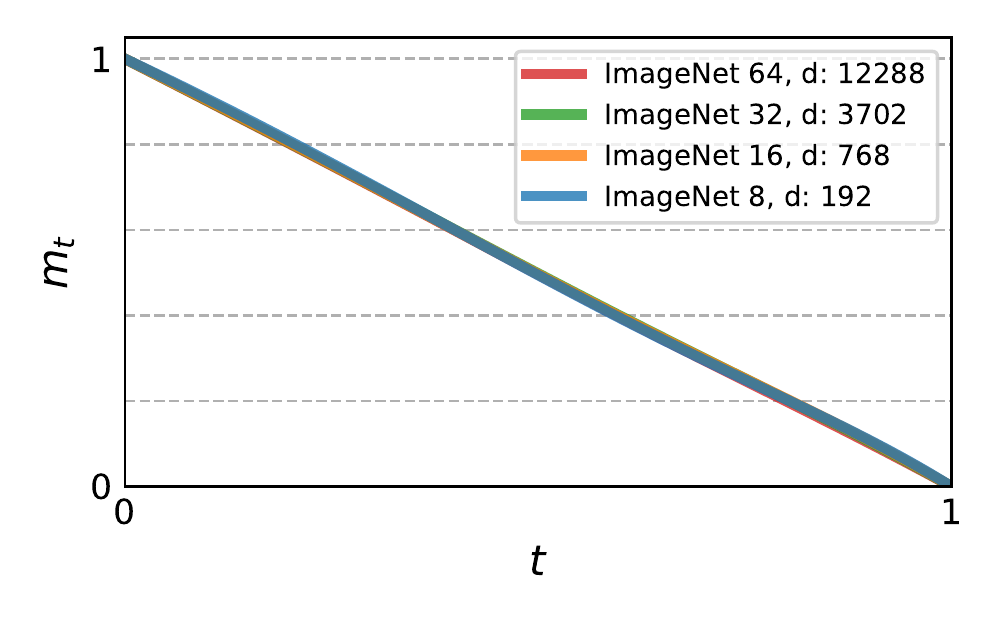} %\vspace{-10pt}
    \end{tabular}
    \caption{(Left) Estimation of the data separation function $\hat{\lambda}$ for ImageNet-8/16/32/64; note the convergence to $1$ as data dimension increases, as anticipated by Theorem \ref{thm:lambda_bound}. (Middle) depicts the optimal $a_t$; and (right) the optimal $m_t$ for ImageNet-8/16/32/64; note convergence to the linear solution (Cond-OT). }
    \label{fig:lambdad}
\end{figure*}

\section{Kinetic optimal paths in high dimensions} 
\label{s:ko_paths_in_high_dim}

Finding kinetic optimal conditional paths $(a_t,m_t)$ in closed form for arbitrary data $q$ is usually hard. However, interestingly, as the dimension of the data increases compared to the number of data samples, a kinetic optimal path can be characterized, at least asymptotically, to be paths defined with the Cond-OT conditional (\eqref{eq:ot_cond}). To facilitate this discussion consider again \eqref{e:ke_explicit} that provides 
\begin{align}\label{e:Ec_E_lambda}
    \hspace{-2pt}\gE_c(a_t,m_t) - \gE(a_t,m_t) &=  \int_0^1  {\beta_t^2 } \parr{\hspace{-2pt}1\hspace{-2pt}-\hspace{-2pt}\lambda\hspace{-2pt}\parr{\frac{a_t}{m_t}}\hspace{-2pt}} dt.
\end{align}
As discussed above and can be seen from this equation when the data separation function achieves its upper bound, $\lambda(s)\equiv 1$, the KE coincides with the CKE, and in this case the optimal path is defined by the Cond-OT conditional (\eqref{eq:ot_cond}). However, it is not a-priori clear if $\lambda$ actually approaches $1$ for any data $q$, and whether closeness of $\lambda$ to $1$ implies Cond-OT is asymptotically kinetic optimal. In this section we show that for \emph{arbitrary} finite normalized dataset consisting of $n$ samples in $\Real^d$, indeed $\lambda\too 1$ (in $L^1$ sense) as $n/\sqrt{d}\too 0$, that can be used to show that the KE of the path defined by the Cond-OT conditional is asymptotically optimal.  In the experiments section we also verify this property empirically for real-world dataset in high dimensions (\ie, ImageNet).

\paragraph{Normalized finite data and main results.} We consider an arbitrary \emph{normalized finite data distribution} $q$. That is, $q$ is distributing equal mass (\ie, $n^{-1}$) among $n$ points $x_{i}\in\Real^d$, $i\in [n]$, and is satisfying equations \ref{e:q_mean} and \ref{e:q_var}, which in the discrete case means
\begin{align}
        & \frac{1}{n}\sum_{i=1}^n x_i = 0,\label{e:q_disc_mean}\\ \label{e:normalization_discrete}
        & \frac{1}{dn}\sum_{i=1}^n\norm{x_i}^2 = 1.
\end{align}
Our first result considers such data distribution $q$ and shows convergence of $\lambda$ to the constant function $1$ as $n/\sqrt{d}\too 0$: 
\begin{theorem}\label{thm:lambda_bound}
Let $q$ be an arbitrary normalized finite data distribution.
Then,
\begin{equation}\label{e:int_abs_1_minus_lambda_bound}
    \int_{0}^{\infty}\abs{1-\lambda(s)}ds \le \frac{3n}{\sqrt{d}}.
\end{equation}
\end{theorem}
This theorem shows that essentially data separation is inevitable in high dimensions and that we can expect $\lambda$ to approach $1$ for finite data in sufficiently high dimension. Using Theorem \ref{thm:lambda_bound} in \eqref{e:Ec_E_lambda} we prove the following theorem, under some mild extra assumption over the family $\gA$ of affine conditional probabilities (\eqref{e:A}): 
\begin{theorem}\label{thm:ke_squeeze}
Let $q$ be an arbitrary normalized finite data distribution. Assume all $(a_t,m_t)\in\gA$ in \eqref{e:A} satisfy: (i) $\frac{a_t}{m_t}$ is strictly monotonically increasing; and (ii) $a_t,m_t$ have uniform bounded Sobolev $W^{1,\infty}([0,1])$ norm\footnote{The Sobolev $W^{1,\infty}([0,1])$ norm is defined as $\norm{f}_{1,\infty} = \max\set{\max_{t\in[0,1]}|f(t)|,\max_{t\in[0,1]}|\dot{f}(t)|}$}, \ie,  $\norm{a_t}_{1,\infty},\norm{m_t}_{1,\infty}\leq M$. Then, for all $(a_t,m_t)\in \gA$ 
\begin{equation}
    \gE_c(a_t,m_t) - 6M^2\frac{n}{\sqrt{d}}\le \gE(a_t,m_t)  \le \gE_c(a_t,m_t).
\end{equation}
\end{theorem}
This shows that the KE of any path defined by a conditional $(a_t,m_t)\in \gA$ will converge to the CKE as $n/\sqrt{d}\too 0$. Since a probability path defined by Cond-OT is minimizing the CKE, we see that Cond-OT is asymptotically kinetic optimal as $n/\sqrt{d}\too 0$. We summarize:
\definecolor{mygray}{gray}{0.95}
\begin{center}\vspace{-10pt}			% Centering minipage
    \colorbox{mygray} {		% Set's the color of minipage
      \begin{minipage}{0.977\linewidth} 	% Starts minipage
       \centering
        \begin{corollary}\label{cor:main}
        The Gaussian probability path defined by the conditional Cond-OT (\eqref{eq:ot_cond}) is asymptotically kinetic optimal as $n/\sqrt{d}\too 0$. 
\end{corollary}
      \end{minipage}}			% End minipage
\end{center}

A comment regrading the first assumption on $\gA$: note that this assumption is equivalent to monotonicity of the signal-to-noise-ratio (SNR), $a_t^2/m_t^2$, which is a natural assumption on conditional paths \cite{kingma2021variational}.

\paragraph{Proof idea.}
The complete proofs for Theorems \ref{thm:lambda_bound} and  \ref{thm:ke_squeeze} are provided in \Cref{aa:data_separation}; here we provide the proof idea for Theorem \ref{thm:lambda_bound}. We first formulate $\lambda$ (\eqref{eq:lambda_s}) for the finite data case. That is 
\begin{equation}\label{eq:lambda_finite}    \lambda(s) = \frac{1}{dn}\sum_{j=1}^n\E_{\rho_s(x|x_j)}\norm{\sum_{i=1}^n x_i\rho_s(x_i|x)}^2,
\end{equation}
where using Bayes' rule we have
\begin{equation*}
    \rho_s(x_i|x) = \frac{\rho_s(x|x_i)q(x_i)}{\sum_{l=1}^nq(x_l)\rho_s(x|x_l)} = \frac{e^{-\frac{s^2}{2}\norm{x_i -x}^2}}{\sum_{l=1}^ne^{-\frac{s^2}{2}\norm{x_l -x}^2}}.
\end{equation*}

That is, $\rho_s(x_i|x)$ is the $i$-th entry of a weighted Softmax (with $s^2$ weight) of the squared Euclidean distances between a point $x$ and all data points $x_i$. The key observation is that for $x\sim\rho_s(x|x_j)$, a noisy sample of $x_j$, the probability $\rho_s(x_i|x)$ is bounded above in expectation by a function that depends only on the pairwise distances, \ie, $\norm{x_i-x_j}$. Furthermore, this function is decaying exponentially with these distances. That is, if the data is ``separated'' in the sense that $\norm{x_i-x_j}$ is large then the data separation function $\lambda$ will be closer to its maximal value of $1$. 
\begin{lemma}\label{lem:key_lem_new}
    Let $q$ be an arbitrary normalized finite data distribution. Then for every $s>0$,
    \begin{align}
           \E_{\rho_s(x|x_j)}\rho_s(x_i|x) \le \eta\parr{s\norm{x_i-x_j}},
    \end{align}
    where $\eta(t)$ is integrable in $[0,\infty)$ and  
    \begin{equation}
        \int_0^{\infty}\eta(t) \le 3.
    \end{equation}
\end{lemma}
This lemma is proved in \Cref{aa:data_separation} as-well. Let us use it to prove our theorem. Using the Cauchy–Schwarz inequality in \eqref{eq:lambda_finite} we can show that 
\begin{align*}
   1-\lambda(s)\hspace{-2pt} \leq  \hspace{-2pt}\frac{1}{dn}\hspace{-2pt}\sum_{i=1}^n \norm{x_i} \sum_{\substack{j=1 \\ j\neq i}}^n \Biggl(\hspace{-2pt}  \norm{x_j-x_i}\E_{\rho_s(x|x_i)}\rho_s(x_j|x)\hspace{-2pt} \biggr)\hspace{-2pt}.
\end{align*}
Now, the integral w.r.t.~$s$ of the term in the parenthesis can be bounded after a change of coordinates $r=s\norm{x_i-x_j}$ using Lemma \ref{lem:key_lem_new} by $3$. Lastly, using Jensen and \eqref{e:normalization_discrete} we have that  $\sum_{i}\norm{x_i}\leq  n\sqrt{d}$, which provides \eqref{e:int_abs_1_minus_lambda_bound}.

%%%%%%%%%%%%%%%%
%2d checkerboard
\begin{figure*}
\centering
\begin{tabular}{@{}c@{\hspace{4pt}}|@{\hspace{4pt}}c@{\hspace{2pt}}c@{\hspace{2pt}}c@{\hspace{2pt}}c@{\hspace{2pt}}c@{}}
    \includegraphics[width=0.68\textwidth]{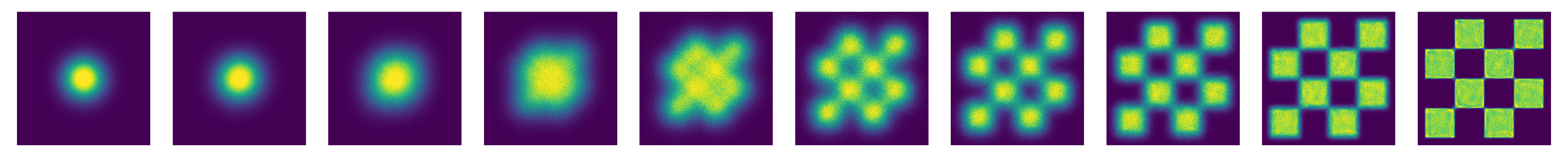} & \rot{\quad \tiny{SI}} &
      \includegraphics[width=0.0632\textwidth]{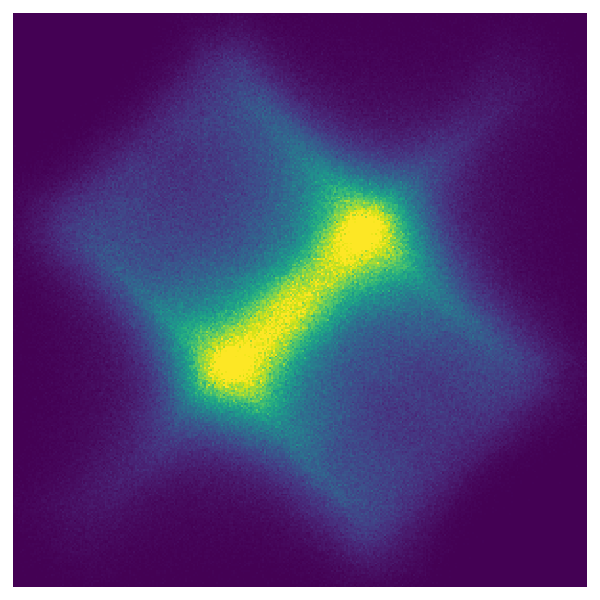} & 
      \includegraphics[width=0.0632\textwidth]{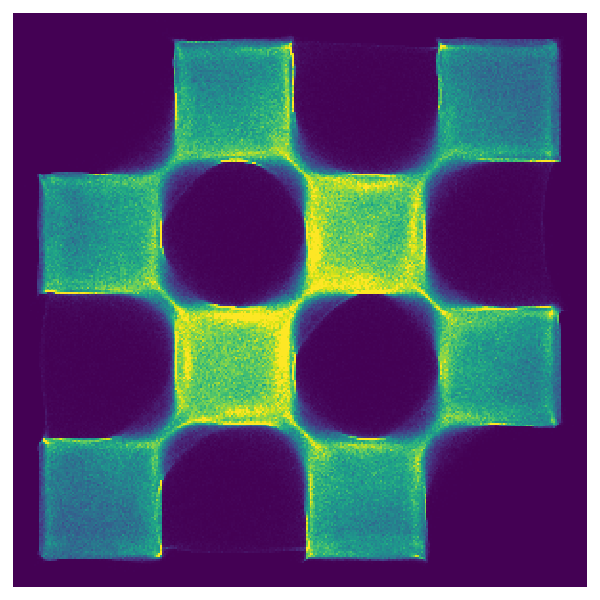} &
      \includegraphics[width=0.0632\textwidth]{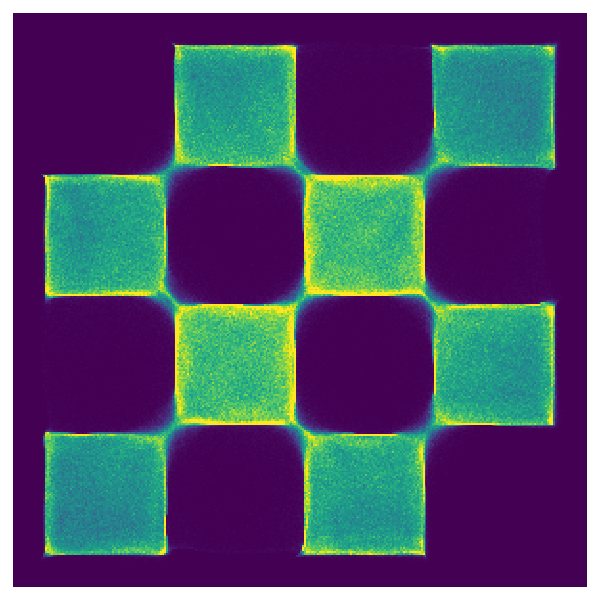} &
      \includegraphics[width=0.0632\textwidth]{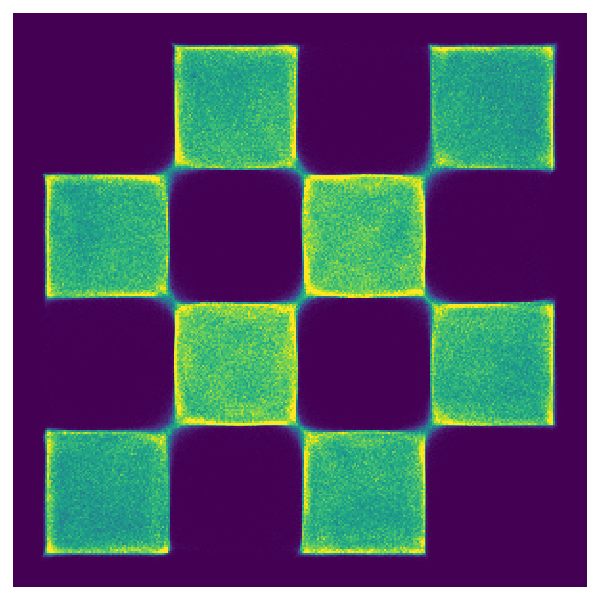}  \vspace{-3pt} \\
      
    \scriptsize{Stochastic Interpolants} & & & & &  \\
    \includegraphics[width=0.68\textwidth]{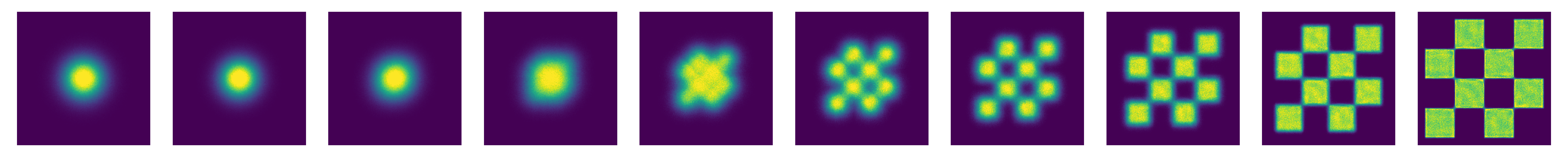} &
    \rot{\tiny{FM \textsuperscript{w}/ COT}} &
      \includegraphics[width=0.0632\textwidth]{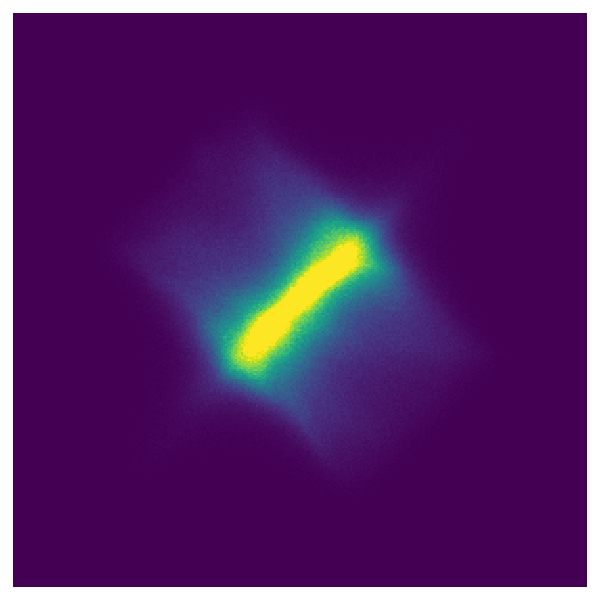} & 
      \includegraphics[width=0.0632\textwidth]{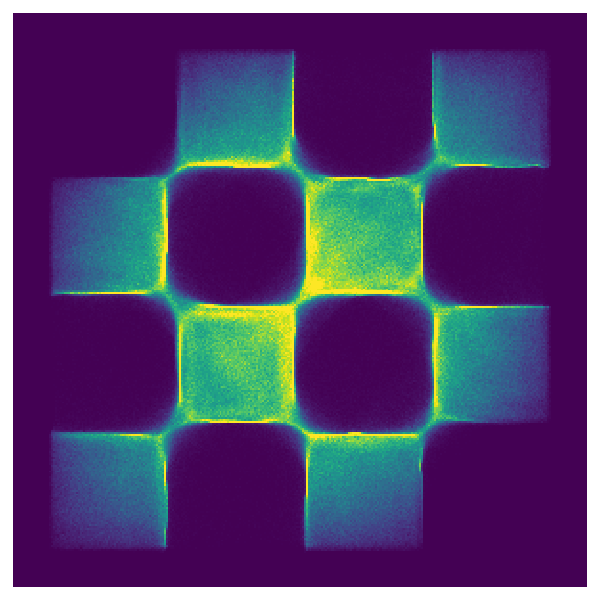} &
      \includegraphics[width=0.0632\textwidth]{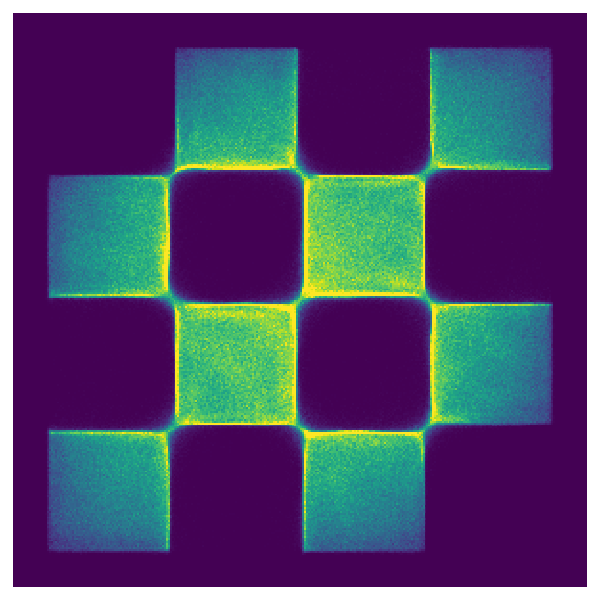} &
      \includegraphics[width=0.0632\textwidth]{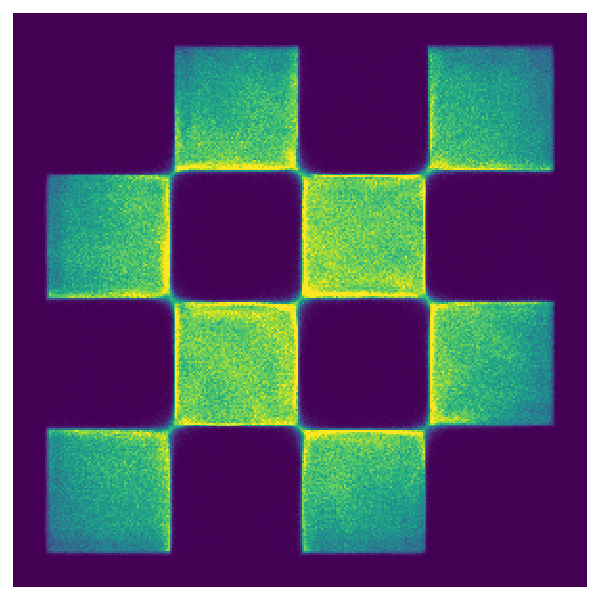} \vspace{-3pt} \\
      
    \scriptsize{Flow Matching \textsuperscript{w}/ Cond-OT } & & & & & \\
    \includegraphics[width=0.68\textwidth]{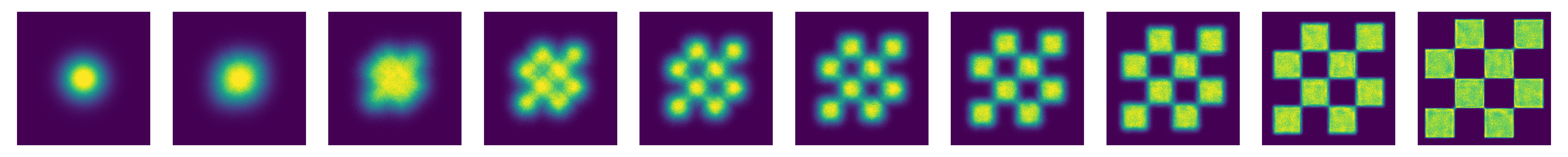} &
    \rot{\tiny{FM \textsuperscript{w}/ KO}} &
      \includegraphics[width=0.0632\textwidth]{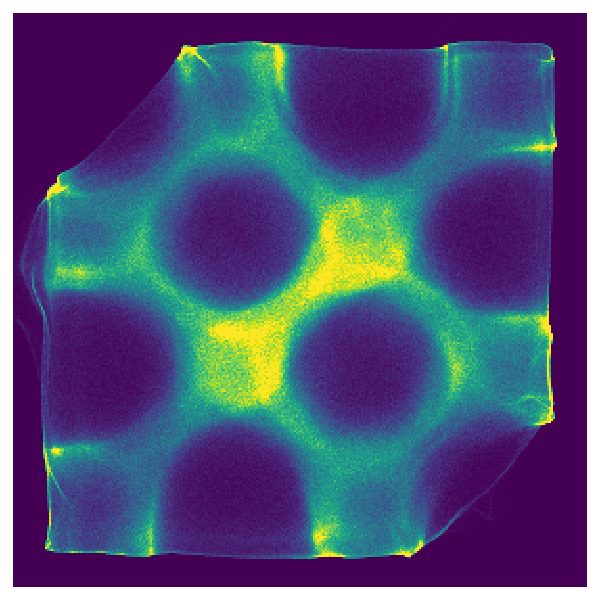} & 
      \includegraphics[width=0.0632\textwidth]{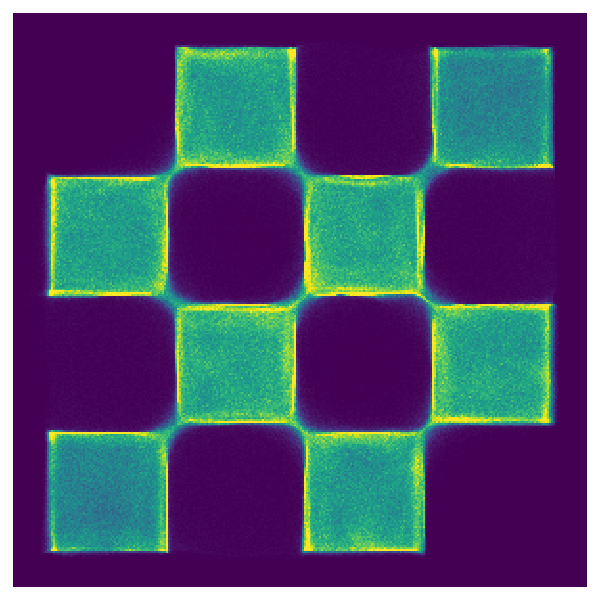} &
      \includegraphics[width=0.0632\textwidth]{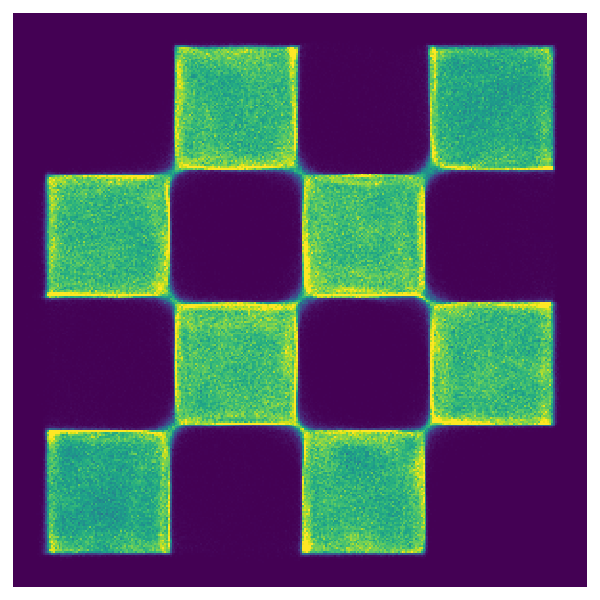} &
      \includegraphics[width=0.0632\textwidth]{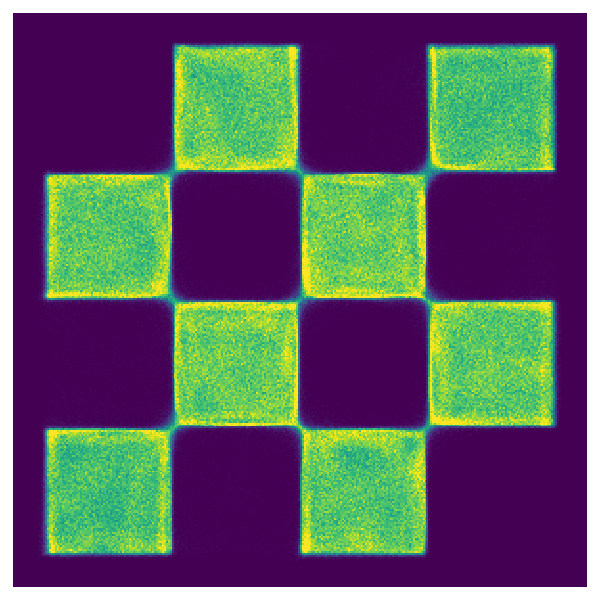} \vspace{-3pt}  \\
      \scriptsize{Flow Matching \textsuperscript{w}/ KO } & & {\scriptsize NFE=2} & {\scriptsize NFE=6} & {\scriptsize NFE=10} & {\scriptsize NFE=16}
\end{tabular}
    \caption{Left: Trajectories of trained CNF models with Stochastic Interpolants, Cond-OT, and Kinetic Optimal paths trained with 2D checkerboard data. Right: Results of the same models when sampling with limited number of function evaluations (NFE). This is a case of low dimensional data (\ie, $d=2$) where Cond-OT is not kinetic optimal. %\vspace{-10pt}
    }
    \label{fig:2d_checkerboard}
\end{figure*}

\section{Related work}
Continuous Normalizing Flows (CNFs) train a flow modeled by learnable vector field $v_t(x)$ \citep{chen2018neural}, where traditionally $v_t(x)$ was optimized to reduce the likelihod of the train data. Other CNF methods, have tried to train CNFs by combining the Kinetic Energy as a regularization term alongside a likelihood term to push the resulting path toward kinetic optimal solutions \citep{finlay2020train,onken2021ot}; however, these still require log probabilities during training and the combination of the losses does not guarantee to produce a globally kinetic optimal solution, which is the Optimal Transport \cite{villani2009optimal}. Computing the log probabilities and their derivatives during training placed a roadblock on scaling CNFs to high dimensions \citep{grathwohl2018ffjord}. Diffusion models \cite{sohl2015deep,ho2020denoising,song2020score} can be seen as an alternative way to train CNFs with a loss that sidesteps log probabilities and regresses the score of an a-priori defined probability density path $p_t \in \gP$ (see \eqref{e:P}). Recently, \citep{lipman2022flow,albergo2022si,liu2022flow,neklyudov2022action} suggest to train a CNF by directly regressing the generating vector field of a probability density paths $p_t\in \gP$. The probability paths utilized in these methods belong to the family of paths considered in this paper, namely the probability paths that marginalize affine conditional per-sample probabilities, and from which we characterize the kinetic optimal paths. Lastly, \cite{albergo2022si} suggested optimizing the kinetic energy simultaneously to training the CNF, resulting in a challenging high dimensional min-max problem; we consider finding the kinetic optimal path used as supervision in the CNF training in separation.

%%%%%%%%%%%%%%%%%%%%%
%figures of the optimal path
\begin{figure*}
    \centering
    \begin{tabular}{@{}c@{\hspace{5pt}}c@{\hspace{2pt}}c@{\hspace{2pt}}c@{}}
         \quad \ \  {\scriptsize 2D }
         & \ \ \
         {\scriptsize CIFAR10}
         &  \ \ 
         {\scriptsize ImagetNet-32}
         & \ \ 
         {\scriptsize ImageNet-64} \\
         \includegraphics[width=0.23\textwidth]{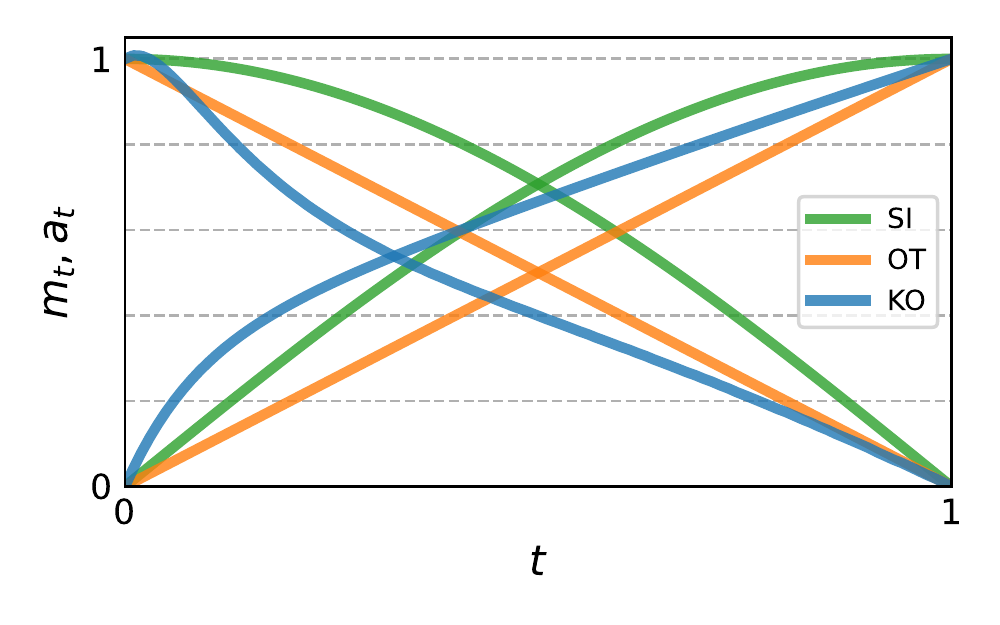} & \includegraphics[width=0.23\textwidth]{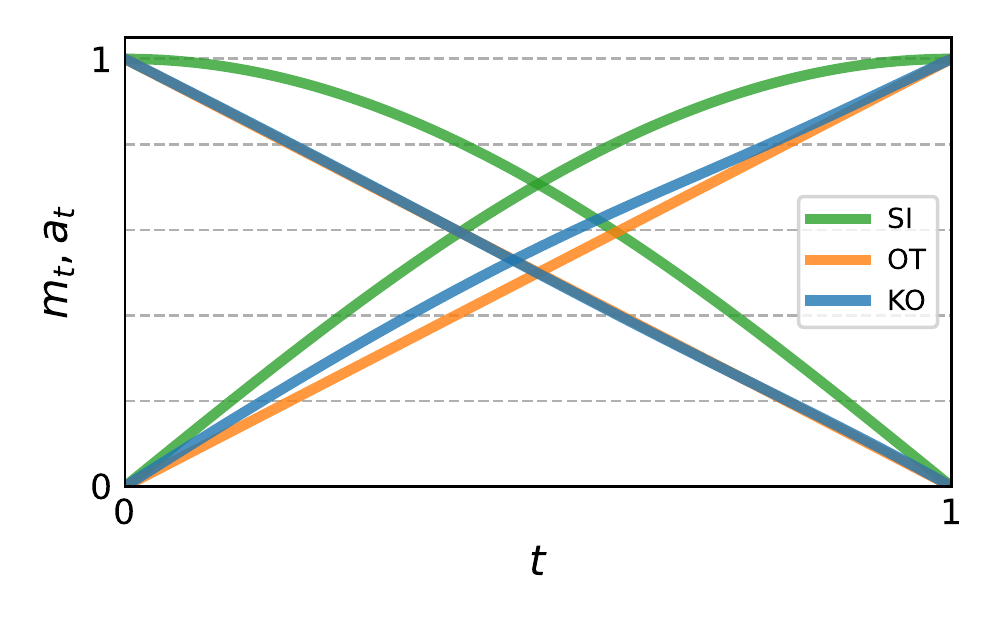} &
         \includegraphics[width=0.23\textwidth]{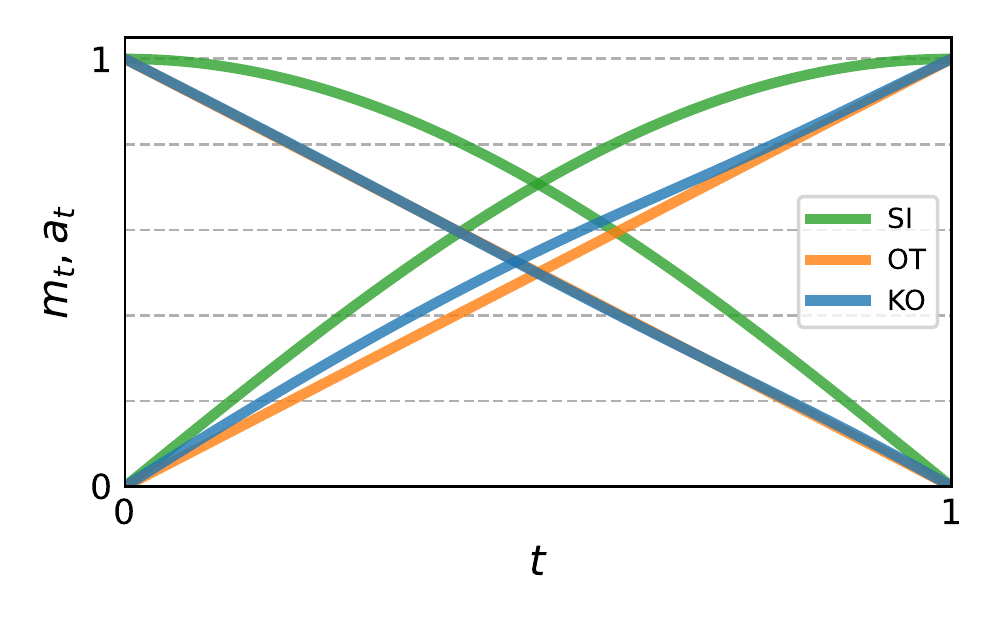} &
         \includegraphics[width=0.23\textwidth]{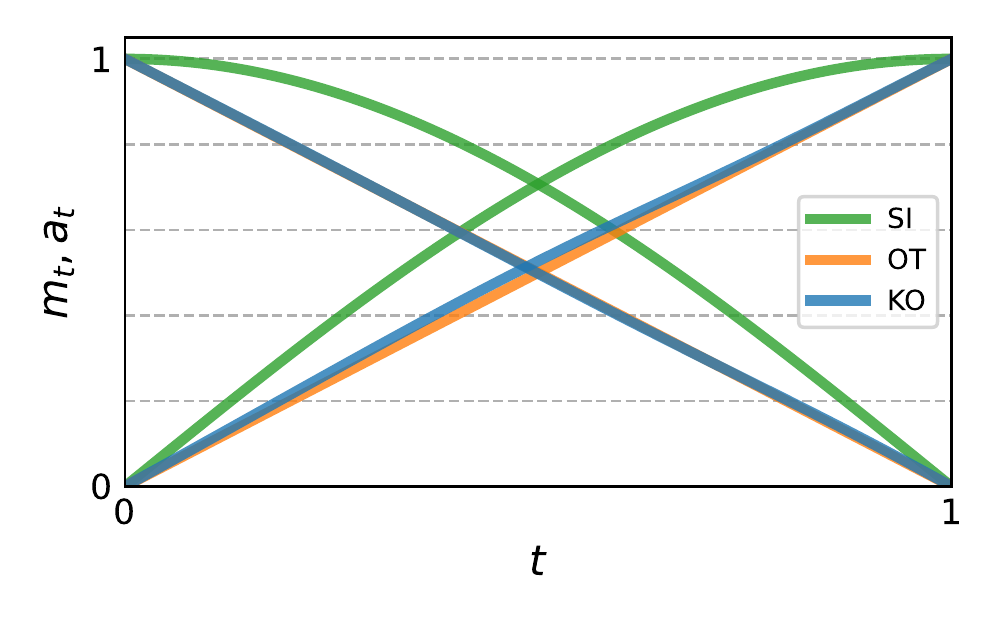} \vspace{-5pt} 
    \end{tabular}
    \caption{The kinetic optimal paths (KO, blue), computed with our algorithm for all datasets considered in the paper, are shown alongside the SI path (green) and Cond-OT path (orange). The KO paths get closer to Cond-OT at high dimensions; for ImageNet-64 the KO and Cond-OT almost coincide.\vspace{10pt} }
    \label{fig:optimal_path}
\end{figure*}

\section{Experiments}
In this section we: (i) approximate the data separation function $\lambda$ numerically for different real-world datasets, and approximate the corresponding Kinetic Optimal (KO) paths for these datasets. (ii) We validate our theory of the convergence of $\lambda\too 1$ in high dimensions for real-world data. (iii) We empirically test our KO paths compared to paths defined by the conditional probabilities of Cond-OT \citep{lipman2022flow}, IS \citep{albergo2022si}, and DDPM \cite{ho2020denoising}. In terms of datasets, we have been experimenting with a 2D dataset (checkerboard), and the image datasets CIFAR10, ImageNet-32, and Imagenet-64.

% Figure of lambda 
\begin{figure}[h!]
    \centering
    \begin{tabular}{@{\hspace{0pt}}c@{\hspace{0pt}}c@{}}         
         \includegraphics[width=0.50\columnwidth]{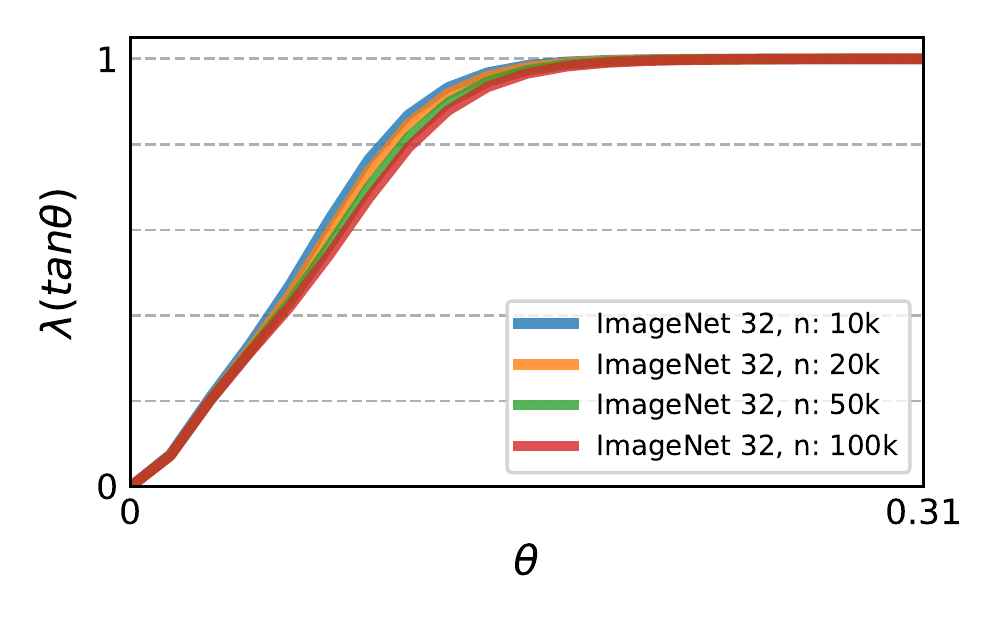} &
         \includegraphics[width=0.50\columnwidth]{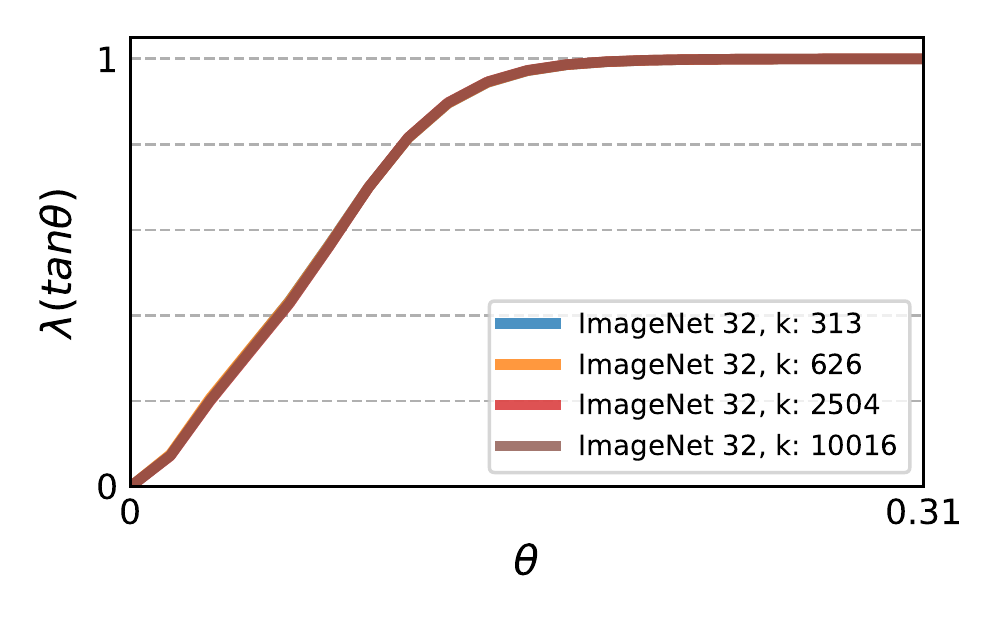} \vspace{-10pt} 
    \end{tabular}
    \caption{Evaluation of data separation estimation: (Left) shows $\hat{\lambda}$ for Imagenet-32 with $n=10K,20K,50K,100K$ data samples, and $k=10K$ samples from $\gN(0,I)$. (Right) shows $\hat{\lambda}$ for Imagenet-32 with $n=50k$ data samples and $k=0.3K,0.6K,2.5K,10K$ samples from $\gN(0,I)$. Note the $x$-axis is limited to $[0,0.31]$ to zoom in on the relevant part. \vspace{-10pt} }
    \label{fig:lambdad_eval}
\end{figure}

\begin{table*}[h]\centering
% \ra{1.05}
\renewcommand{\tabcolsep}{3.pt}
\resizebox{\textwidth}{!}{
\begin{tabular}{l c c c c c c c c c c c c c c c  }\toprule %r r r r r r r r r r r r r r r
  & \multicolumn{7}{c}{\bf CIFAR-10} &
  & \multicolumn{7}{c}{\bf ImageNet 32$\times$32} \\
\cmidrule(lr){2-8} \cmidrule(lr){10-16} 
Model & {NLL$\downarrow$} & {KE$\downarrow$} & {FID$_{20}$$\downarrow$} & {FID$_{50}$$\downarrow$} & {FID$_{100}$$\downarrow$} & {FID$_{Adpt}$$\downarrow$} & {NFE$_{Adpt}$$\downarrow$}
& & {NLL$\downarrow$} & {KE$\downarrow$} & {FID$_{20}$$\downarrow$} &{FID$_{50}$$\downarrow$} & {FID$_{100}$$\downarrow$} &{FID$_{Adpt}$$\downarrow$} & {NFE$_{Adpt}$$\downarrow$} \\
\cmidrule(r){1-1}\cmidrule(lr){2-8} \cmidrule(lr){10-16}
\textit{\small Ablations}\\
\;\; DDPM   & 
$3.11$ & $2.30$ & $38.5 \pm 0.2$ & $18.0 \pm 0.0$ & $12.5 \pm 0.1$ & $7.4 \pm 0.1$ & $344 \pm 0$ & & 
$3.61$ & $2.22$ & $66.1 \pm 1.2$ & $15.7 \pm 0.3$& $8.6 \pm 0.3$ & $5.1 \pm 0.1$ & $333 \pm 2$ \\
\;\; SI & %~{\tiny\citep{song2020score}} & 
$3.00$ & $1.30$ & $10.4 \pm 0.1$ & $7.4 \pm 0.1$ & $7.4 \pm 0.1$ & $7.5 \pm 0.0$ & $ 165 \pm 0$ & & 
$3.58$ & $1.27$ & $23.4 \pm 0.4$ & $9.0 \pm 0.3$ & $6.7 \pm 0.2$ & $5.3 \pm 0.2$ & $158 \pm 2$ \\
\cmidrule(r){1-1}\cmidrule(lr){2-8} \cmidrule(lr){10-16}
\;\; FM \textsuperscript{w}/ OT path    & 
$2.97$ & $1.09$ & $9.9 \pm 0.0$ & $6.8 \pm 0.0$ & $6.8 \pm 0.0$ &$6.9 \pm 0.0$ & $146 \pm 0$ & &  
$3.60$ & $1.08$ & $9.7 \pm 0.1$ & $6.5 \pm 0.1$ & $5.6 \pm 0.0$ &$5.0 \pm 0.1$ & $145 \pm 0$ \\

%\textit{\small Ours}\\
\;\; FM \textsuperscript{w}/ KO path & 
$2.97$ & $\bf{1.08}$ & $9.1 \pm 0.0$ & $6.2 \pm 0.1$ & $6.2 \pm 0.1$ &$ 6.1 \pm 0.0$ & $147 \pm 0$ & & 
$3.57$ & $\bf{1.07}$ & $9.8 \pm 0.3$ & $ 6.6 \pm 0.4$ & $5.8 \pm 0.3$ &$5.1 \pm 0.3$ & $129 \pm 1$ \\
\bottomrule
\end{tabular}
}

\caption{For the datasets of CIFAR-10 and ImageNet-32 we log: Negative log likelihoods in BPD units; Kinetic Energy (KE); quality of generated samples in FID computed with Euler method for 20,50,100 steps and the adaptive method DOPRI5; and Number of Function Evaluation (NFE) with the adaptive solver. The KE, FID and NFE are averaged over three checkpoints (epochs 1990, 1995, 2000 for CIFAR-10; and epochs 390, 395, 400 for ImageNet-32). \vspace{+5pt} }
\label{tab:img_results}
\end{table*}

\subsection{Approximation of data separation function}\label{ss:approx_of_lambda}
A simple approach to approximate $\lambda$ in \eqref{eq:lambda_s} for a given dataset $x_i\in\Real^d$, $i\in [n]$, is to consider normalized finite data $q$ (see Section \ref{s:ko_paths_in_high_dim}), \ie, giving each $x_i$ the probability $n^{-1}$, normalized in the sense of \eqref{e:q_mean} and \eqref{e:q_var}. In this case $\lambda$ takes the form of \eqref{eq:lambda_finite}, and we use Monte-Carlo estimation of the expectation w.r.t.~$x\sim \rho_s(x|x_j)$, $j\in [n]$, where we can express $x$ as $x = x_j + s^{-1} z$, and $z\sim \gN(0,I)$. We sample $z_l \sim \gN(0,I)$, $l\in [k]$ i.i.d.~and our estimator becomes
\begin{equation}\label{eq:lambda_estimator}
    \hat{\lambda}(s) = \frac{1}{nd}\sum_{l=1}^k\sum_{j=1}^{n}\norm{\sum_{i=1}^{n} x_i\rho_s(x_i|x_j +s^{-1}z_l)}^2.
\end{equation}
Unfortunately, computing this sum requires the pairwise distances $\norm{x_i-x_j}^2$ and therefore direct computation scales quadratically in $n$. To visualize $\hat{\lambda}:[0,\infty)\too [0,1]$ compactly we instead show $\hat{\lambda}\circ \tan:[0,\pi/2]\too [0,1]$, so the $x$-axis is annotated with $\theta$. Figure \ref{fig:lambdad_eval} (left) shows the difference in the estimated lambda (computed with \eqref{eq:lambda_estimator}) for ImageNet-32 using an increasing number of samples from $q$ starting at $n=10K$ to $n=100K$ while keeping the number of samples from $\gN(0,I)$ constant at $k=10K$. As can be inspected in this figure, beyond $50K$ the differences are rather small. Note that we display the range $\theta\in [0,0.3]$ since beyond that all $\hat{\lambda}$ approximations practically reach $1$, which is also the upper bound for $\lambda$. Similarly, in Figure \ref{fig:lambdad_eval} (right), we keep the number of samples from $q$ constant at $n=50K$ and vary the number of samples $k$ from $\gN(0,I)$, $k\in [0.3K,10K]$, again noticing that the differences between the estimators are negligible. 
We approximate $\hat{\lambda}(\tan(\theta))$ on equispaced grid points $\theta_i = 0.01\frac{\pi}{2}i$, $i\in [0,100]$ and fit a cubic-spline to this data to be used for estimating $\hat{\lambda}(\tan \theta)$ at arbitrary $\theta\in[0,\frac{\pi}{2}]$. 

\paragraph{Optimizing for kinetic optimal paths.} Given an estimator of $\hat{\lambda}$, we minimize the KE energy in \eqref{eq:ke_polar_form} as a functional of $(r_t,\theta_t)$. We model $r_t$ with the derived one-parameter solutions family in Theorem \ref{thm:main_optimal_path}. For $\theta_t$, we use a 10 parameters neural network, and optimize the KE with general purpose gradient descent, see details in \Cref{a:implementation-details}. Lastly, we apply the coordinate change in \eqref{e:polar_change} and transform the optimal $r_t,\theta_t$ back to $a_t,m_t$ that is used for training. Figure \ref{fig:optimal_path} shows the kinetic optimal paths computed for all datasets considered in the paper. It is already obvious in this figure that as the dimension of the data increases the kinetic optimal paths get closer to the Cond-OT path, anticipated in Theorem \ref{thm:ke_squeeze}. We discuss this next.

\subsection{Kinetic optimal paths in high dimensions}
The high dimensional phenomenon presented in Theorem \ref{thm:lambda_bound}, namely that $\lambda\too 1$ for discrete data in high dimensions, does in fact manifests in real-world high dimensional datasets such as ImageNet, even though the ratio $\frac{n}{\sqrt{d}}$ is not small for the relevant dimensions. Figure \ref{fig:lambdad} (left) shows $\hat{\lambda}$ for ImageNet-8/16/32/64, which have dimensions 192/768/3072/12,228 (resp.) computed with $n=50K$ and $k=10K$. We provide a detailed run-times discussion in Appendix \ref{a:runtime}. As can be seen in this figure, as the image dimension increases indeed $\lambda\too 1$. In particular, for ImageNet-64, $\hat{\lambda}$ converged to $1$ for every $\theta \in [0.08, \frac{\pi}{2}]$. (Middle),(Right) Show the corresponding kinetic optimal paths $a_t$, $m_t$ (resp.); the  convergence of the KO paths to Cond-OT as data dimension increases is suggested by theorem \ref{thm:ke_squeeze} and the fact that Cond-OT is the minimizer of the CKE. In particular, the KO optimal path for ImageNet-64 is practically identical to the Cond-OT linear path.

\subsection{Flow Matching with kinetic optimal paths}\label{ss:FM_with_KO} 
We also experimented with kinetic optimal (KO) paths for training a continuous normalizing flows. We use the Flow Matching (FM) approach \citep{lipman2022flow} where (in the notations of \eqref{e:loss}),
\begin{equation}\label{e:fm_loss}
    \gL(\theta)  = \E_{t,p_t(x)} \norm{v_t(x;\theta) - u_t(x)}^2,
\end{equation}
where $v_t$ is a neural network with learnable parameters $\theta$, $p_t(x)$, $u_t(x)$ are as defined in equations \ref{e:p_t}, \ref{e:u_t} (resp.) and $t$ is uniform in $[0,1]$. The actual training objective is
\begin{equation}\label{e:cfm_loss}
    \gL(\theta) = \E_{t,p_t(x|x_1),q(x_1)}\norm{v_t(x;\theta)-u_t(x|x_1)}^2,
\end{equation}
where $p_t(x|x_1)$ is as defined in \eqref{eq:p_cond} and $u_t(x|x_1)$ as defined in \ref{e:u_t_cond}, which is shown in \citep{lipman2022flow} to have equivalent gradients to the loss in \eqref{e:fm_loss}. For $a_t,m_t$ we use the optimal solutions of the approximated kinetic energy (see Section \ref{ss:approx_of_lambda}); see Figure \ref{fig:optimal_path} for visualizations of these KO paths.

\paragraph{2D data.}
We tested FM-KO first on 2D data of a checkerboard distribution, often used for low dimensional test-bed for CNF models. Figure \ref{fig:2d_checkerboard} compares FM-KO, with FM-Cond-OT, and IS: On the left we show the time sequence, $t\in [0,1]$, of the generated probabilities starting from the standard noise $p_0=p$. Note that KO is distributed more evenly in time. On the right we show sampling with limited number of function evaluations (NFE) using the Euler solver. Note that KO outperforms the baselines for low NFE counts. In this case the data dimension is low ($d=2$) and the dataset contains many (more then 1M) data points, which is outside the scope of our Theory that is effective as $n/\sqrt{d}$ is smaller. Indeed, in this case Cond-OT is not kinetic optimal and is being outperformed in terms of NFE counts.

% Figure of image vs nfe 
\begin{figure*}[h]
    \centering
    \begin{tabular}{@{\hspace{0pt}}c@{\hspace{8pt}}c@{\hspace{8pt}}c@{\hspace{0pt}}}
         \multicolumn{1}{p{140pt}}{\scriptsize \hspace{-11pt}NFE:\hspace{2pt}6\hspace{26pt}8\hspace{24pt}12\hspace{22pt}20\hspace{20pt}100}&\multicolumn{1}{p{140pt}}{\scriptsize \hspace{-11pt}NFE:\hspace{2pt}6\hspace{26pt}8\hspace{24pt}12\hspace{22pt}20\hspace{20pt}100}& \multicolumn{1}{p{140pt}}{\scriptsize \hspace{-10pt}NFE:\hspace{2pt}6\hspace{26pt}8\hspace{24pt}12\hspace{22pt}20\hspace{20pt}100}\\
         \includegraphics[width=0.30\textwidth]{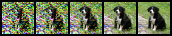} & \includegraphics[width=0.30\textwidth]{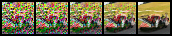} &
         \includegraphics[width=0.30\textwidth]{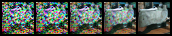} \vspace{-3pt} \\
         & \scriptsize{DDPM} & \vspace{1pt} \\
         \includegraphics[width=0.30\textwidth]{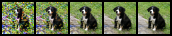} & \includegraphics[width=0.30\textwidth]{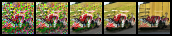} &
         \includegraphics[width=0.30\textwidth]{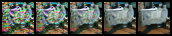} \vspace{-2pt} \\
         & \scriptsize{Stochastic Interpolants} & \vspace{1pt}\\
         \includegraphics[width=0.30\textwidth]{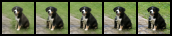} & \includegraphics[width=0.30\textwidth]{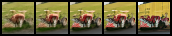} &
         \includegraphics[width=0.30\textwidth]{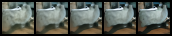} \vspace{-2pt}\\
         & \scriptsize{Flow Matching \textsuperscript{w}/ Cond-OT } & \vspace{1pt} \\
         \includegraphics[width=0.30\textwidth]{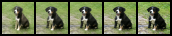} & \includegraphics[width=0.30\textwidth]{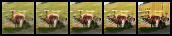} &
         \includegraphics[width=0.30\textwidth]{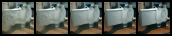} \vspace{-2pt} \\
         & \scriptsize{Flow Matching \textsuperscript{w}/ KO } & \vspace{1pt} \\
         \vspace{-15pt}
    \end{tabular}    \caption{Generated images of CNF models trained with DDPM, SI, Cond-OT and Kinetic Optimal paths on \textit{face-blurred} ImageNet-32 data. Generation is done with the Euler method, where in each column all images are generated from the same noise but with a different number of steps (equivalent to NFE for Euler solver) as indicated; 6, 8, 12, 20 and 100 steps.\vspace{-10pt}} 
    \label{fig:imagenet32_image_vs_nfe}
\end{figure*}

%\columnbreak
\paragraph{Image datasets.}
Lastly, we tested KO optimal paths on CIFAR and ImageNet-32 (for ImageNet-64 KO paths practically coincides with Cond-OT). For ImageNet we use the official \textit{face-blurred} ImageNet and downsample to $32 \times 32$ using an open source preprocessing script \cite{chrabaszcz2017downsampled}. Quantitative results are summarizes in Table \ref{tab:img_results} while qualitative results are depicted in Figure \ref{fig:imagenet32_image_vs_nfe}; for each method we evaluated Negative Log Likelihoods in Bit-Per-Dimension (BPD); kinetic energy; and FID computed for samples generated by Euler method with 20, 50, 100 steps and the adaptive solver DOPRI5. For FID and NFE (adaptive) we averaged $3$ different epochs at the final stage of training. For these datasets, we find that KO and Cond-OT paths already produce similar KE, better than the other baselines considered. Figure \ref{fig:fid_nfe} depicts curves of generated image quality (FID) as a function of number of function evaluation (NFE) for CIFAR10 ImageNet-32. We provide example of generated images from the KO model in Appendix \ref{a:additional_figures}. 

\begin{figure}
    \centering
    \begin{tabular}{@{\hspace{0pt}}c@{\hspace{0pt}}c@{}}
    \quad \ \  {\scriptsize CIFAR10 }
         & \ \ \
         {\scriptsize ImageNet 32}\\
         \includegraphics[width=0.50\columnwidth]{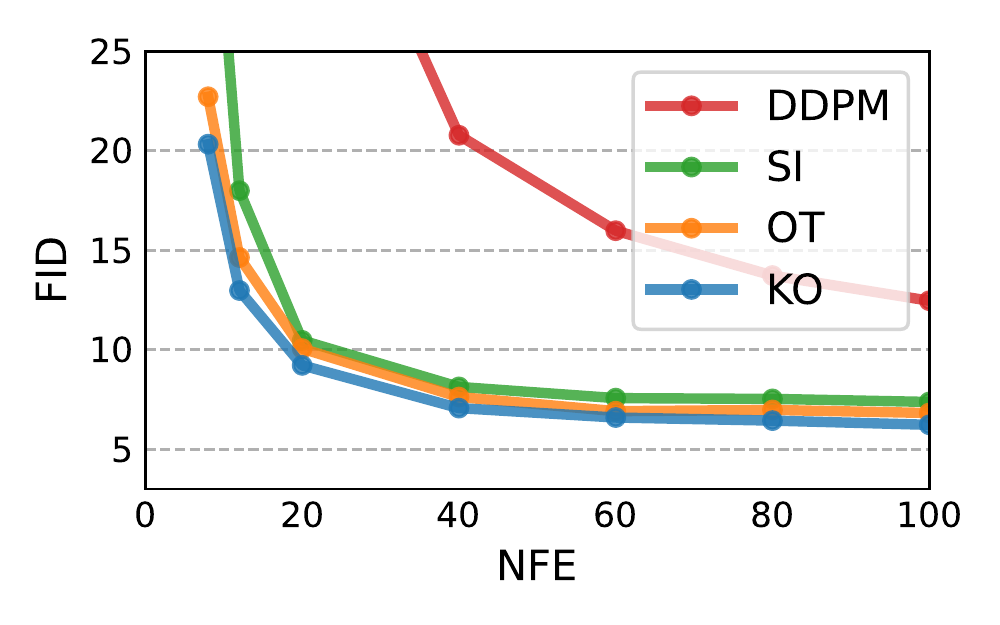} &
         \includegraphics[width=0.50\columnwidth]{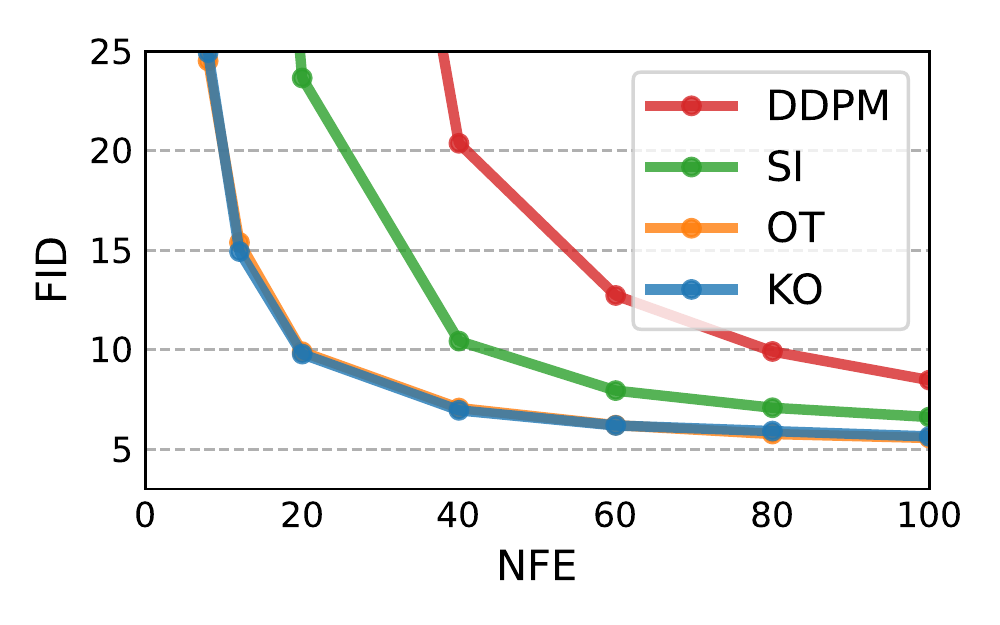} \vspace{-10pt} 
    \end{tabular}
    \caption{FID versus NFE for models trained on the CIFAR-10 (left) and ImageNet-32 (right) datasets. Samples for FID computation are generated with Euler solver.} \vspace{-16pt}
    \label{fig:fid_nfe} 
\end{figure}

\section{Conclusions}
In this paper we investigated the space of tractable probability paths, which are used to supervise generative models' training, and searched for an optimal path that minimizes the Kinetic Energy. We started with identifying a more tractable form for the kinetic energy that incorporates the training data using a simple, one dimensional data separation function $\lambda$. We then characterized the kinetic optimal solutions and suggested a method to compute them using a numerical estimation to $\lambda$. The kinetic optimal paths improved performance of the respectively trained generative models mostly in low to medium data dimensions. Lastly, we demonstrated through a theoretical analysis accompanied with empirical experiments that as the dimension of the data increases, the simple Cond-OT probability paths are becoming kinetic optimal. That is, the kinetic energy of paths defined with the Cond-OT conditional is converging to the optimal kinetic energy in a family of probability paths. This result has several implications for future research: First, further searching the class Gaussian paths defined by affine conditionals (equations \ref{e:p_t}, \ref{eq:p_cond} and \ref{e:A}) for useful probability path might be futile. In fact, in high dimensions, Cond-OT might be a close to optimal choice within this class, for most datasets.  
Finding probability path with even lower Kinetic Energy would entail abandoning the affine conditional probabilities and searching within a broader class. For example, replacing the affine model (\eqref{eq:p_cond}) with a model that has more complex dependency on $x_1$ might allow further improvement of the kinetic energy and consequently the trained generative model. 

\subsection*{Acknowledgements}

NS was supported by a grant from Israel CHE Program for
Data Science Research Centers.

% Acknowledgements should only appear in the accepted version.
% \section*{Acknowledgements}

% \textbf{Do not} include acknowledgements in the initial version of
% the paper submitted for blind review.

\bibliography{kinetic}
\bibliographystyle{icml2023}

%%%%%%%%%%%%%%%%%%%%%%%%%%%%%%%%%%%%%%%%%%%%%%%%%%%%%%%%%%%%%%%%%%%%%%%%%%%%%%%
%%%%%%%%%%%%%%%%%%%%%%%%%%%%%%%%%%%%%%%%%%%%%%%%%%%%%%%%%%%%%%%%%%%%%%%%%%%%%%%
% APPENDIX
%%%%%%%%%%%%%%%%%%%%%%%%%%%%%%%%%%%%%%%%%%%%%%%%%%%%%%%%%%%%%%%%%%%%%%%%%%%%%%%
%%%%%%%%%%%%%%%%%%%%%%%%%%%%%%%%%%%%%%%%%%%%%%%%%%%%%%%%%%%%%%%%%%%%%%%%%%%%%%%
\newpage
\appendix
\onecolumn

\section{Auxiliary computations and proofs}

\subsection{Kinetic energy and $\lambda$}
\label{aa:ke_comp}

\begin{proof}[Proof of \eqref{e:ke_cond_closed_form}.]
We compute the Conditional Kinetic Energy (CKE) $\gE_c(a_t,m_t)$. By definition, it is
\begin{equation}
    \gE_c(a_t,m_t) = \frac{1}{d} \E_{t,q(x_1),p_t(x|x_1)}\norm{u_t(x|x_1)}^2,
\end{equation}
where $u_t(x|x_1)$ is as defined in \eqref{e:u_t_cond}. That is
\begin{equation}
    u_t(x|x_1) = \alpha_t x + \beta_t x_1
\end{equation}
and 
\begin{equation}
    \alpha_t = \frac{\dot{m}_t}{m_t}, \quad \beta_t = \dot{a}_t-a_t \frac{\dot{m}_t}{m_t}.
\end{equation}
We remind that $p_t(x|x_1) = \gN(x|a_tx_1,m_t^2I)$ and $p(x) = \gN(x|0,I)$. Hence, for $x\sim p(x)$, $m_tx +a_tx_1 \sim p_t(x|x_1)$. We can now compute the expectation w.r.t.~$p_t(x|x_1)$ and $q(x_1)$. That is,
\begin{align}\label{e:cke_eval}
    \frac{1}{d} \E_{t,q(x_1),p_t(x|x_1)}\norm{u_t(x|x_1)}^2 &= \frac{1}{d} \E_{t,q(x_1),p(x)}\norm{u_t(m_tx+a_tx_1|x_1)}^2\\ 
    &= \frac{1}{d} \E_{t,q(x_1),p(x)}\norm{\dot{m}_t x + \dot{a}_t x_1}^2\\
    &= \frac{1}{d} \E_{t,q(x_1),p(x)}\parr{\dot{m}_t^2 \norm{x}^2 +2 \dot{m}_t\dot{a}_t x\cdot x_1 + \dot{a}_t^2 \norm{x_1}^2}\\
    &= \frac{1}{d} \E_{t,q(x_1)}\parr{\dot{m}_t^2d + \dot{a}_t^2 \norm{x_1}^2}\\
    &= \frac{1}{d} \E_{t}\parr{\dot{m}_t^2d + \dot{a}_t^2 d}\\
    &= \int_0^1 (\dot{m_t}^2 +\dot{a_t}^2) dt,
\end{align}
where before the last equality we used \eqref{e:q_var} to compute the expectation on $q(x_1)$.

\emph{Proof of equations \ref{e:ke_explicit} and \ref{eq:lambda_s}.}
We compute the Kinetic Energy (KE) $\gE(a_t,m_t)$, proving \eqref{e:ke_explicit}. By definition, it is
\begin{equation}
    \gE(a_t,m_t) = \frac{1}{d}\E_{t,p_t(x)} \norm{u_t(x)}^2,
\end{equation}
where $u_t(x)$ is as defined in \eqref{e:u_t}. Equivalently, using Bayes' rule, we can write it as
\begin{equation}
     u_t(x) = \E_{p_t(x_1|x)}u_t(x|x_1)
\end{equation}
and 
\begin{equation}
    p_t(x_1|x) = \frac{p_t(x|x_1) q(x_1)}{p_t(x)}.
\end{equation}
Now we can calculate the KE as follows
\begin{align}\label{e:ke_eval}
     \frac{1}{d}\E_{t,p_t(x)} \norm{u_t(x)}^2 &=\frac{1}{d}\E_{t,p_t(x)} \norm{\E_{p_t(x_1|x)}u_t(x|x_1)}^2\\
     &= \frac{1}{d}\E_{t,p_t(x)} \norm{\alpha_tx + \beta_t\E_{p_t(x_1|x)}x_1}^2\\
     &= \frac{1}{d}\E_{t,p_t(x)} \parr{\alpha_t^2\norm{x}^2 +2\alpha_t\beta_tx\cdot \E_{p_t(x_1|x)}x_1  + \beta_t^2\norm{\E_{p_t(x_1|x)}x_1}^2}\\
     &= \frac{1}{d}\E_{t,p_t(x)} \parr{\alpha_t^2\norm{x}^2 +2\alpha_t\beta_tx\cdot \E_{p_t(x_1|x)}x_1  +\beta_t^2d -\beta_t^2d + \beta_t^2\norm{\E_{p_t(x_1|x)}x_1}^2}\\
     &= \frac{1}{d}\E_{t,p_t(x)} \parr{\alpha_t^2\norm{x}^2 +2\alpha_t\beta_tx\cdot \E_{p_t(x_1|x)}x_1  +\beta_t^2d} - \frac{1}{d}\E_{t,p_t(x)}\parr{\beta_t^2d - \beta_t^2\norm{\E_{p_t(x_1|x)}x_1}^2}
\end{align}
We perform the expectation over $p_t(x)$ for each term on the r.h.s.~separately. We invoke Bayes' rule again, which gives
\begin{equation}
    p_t(x)p_t(x_1|x) = q(x_1)p_t(x|x_1).
\end{equation}
Hence, the first term is
\begin{align}
    \frac{1}{d}\E_{t,p_t(x)} \parr{\alpha_t^2\norm{x}^2 +2\alpha_t\beta_tx\cdot \E_{p_t(x_1|x)}x_1  +\beta_t^2d} &= \frac{1}{d}\E_{t,p_t(x),p_t(x_1|x)} \parr{\alpha_t^2\norm{x}^2 +2\alpha_t\beta_tx\cdot x_1  +\beta_t^2d}\\
    &= \frac{1}{d}\E_{t,q(x_1),p_t(x|x_1)} \parr{\alpha_t^2\norm{x}^2 +2\alpha_t\beta_tx\cdot x_1  +\beta_t^2d}\\
    &= \frac{1}{d}\E_{t,q(x_1),p_t(x|x_1)} \parr{\alpha_t^2\norm{x}^2 +2\alpha_t\beta_tx\cdot x_1  +\beta_t^2\norm{x_1}^2}\\
    &= \frac{1}{d}\E_{t,q(x_1),p_t(x|x_1)}\norm{\alpha_tx + \beta_tx_1}^2\\
    &= \frac{1}{d}\E_{t,q(x_1),p_t(x|x_1)}\norm{u_t(x|x_1)}^2\\
    &= \gE_c(a_t,m_t),
\end{align}
where we also used \eqref{e:q_var} for the expectation of $\norm{x_1}^2$. Next, we define $\lambda$ 
\begin{equation}
    \lambda\parr{\frac{a_t}{m_t}} = \frac{1}{d}\E_{p_t(x)}\norm{\E_{p_t(x_1|x)}x_1}^2.
\end{equation}
So the second term is 
\begin{align}
    \frac{1}{d}\E_{t,p_t(x)}\parr{\beta_t^2d - \beta_t^2\norm{\E_{p_t(x_1|x)}x_1}^2} &= \E_{t}\parr{\beta_t^2 - \beta_t^2\lambda\parr{\frac{a_t}{m_t}}}\\
    &= \int_0^1\beta_t^2\parr{1 - \lambda\parr{\frac{a_t}{m_t}}}dt.
\end{align} 
Finally, substituting the two terms, gives
\begin{align}
    \gE(a_t,m_t) &= \gE_c(a_t,m_t) - \int_0^1  {\beta_t^2 } \parr{1-\lambda\parr{\frac{a_t}{m_t}}} dt.
\end{align}
It is left to show that, indeed, we can write $\lambda$ as a one-variable function. Writing $\lambda$ explicitly gives
\begin{align}
    \frac{1}{d}\E_{p_t(x)}\norm{\E_{p_t(x_1|x)}x_1}^2 &= \frac{1}{d}\int \norm{\int\frac{\tilde{x}_1\gN\parr{x|a_t\tilde{x}_1, m_t^2I}q(\tilde{x}_1)d\tilde{x}_1}{\int \gN\parr{x|a_t\hat{x}_1, m_t^2I}q(\hat{x}_1)d\hat{x}_1}}^2\gN\parr{x|a_tx_1, m_t^2I}q(x_1)dx_1dx.
\end{align}
We will use the following identity: Given $\mu\in \Real^d$ and $\Sigma\in S_+^{d}$, 
\begin{align}
    \gN(x|\mu,\Sigma) &= (\det A) \gN(Ax | A\mu, A\Sigma A^{T}). 
\end{align}
For $A=\frac{1}{a_t}$ and the change of variable $\frac{x}{a_t}\mapsto x$ we get the r.h.s.~below which justifies writing $\lambda$ as a function of $a_t/m_t$:
\begin{align}
    \lambda\parr{\frac{a_t}{m_t}} &= \frac{1}{d}\int \norm{\int\frac{\tilde{x}_1\gN\parr{x\bigg|\tilde{x}_1, \parr{\frac{m_t}{a_t}}^2I}q(\tilde{x}_1)d\tilde{x}_1}{\int \gN\parr{x\bigg|\hat{x}_1, \parr{\frac{m_t}{a_t}}^2I}q(\hat{x}_1)d\hat{x}_1}}^2\gN\parr{x\bigg|x_1, \parr{\frac{m_t}{a_t}}^2I}q(x_1)dx_1dx.
\end{align}
\end{proof}

\newpage
\subsection{The $\gamma(\theta_t)>0$ condition}
\label{aa:gamma_cond}
\begin{proof}[Proof of $\gamma(\theta_t)>0$ condition.]
We show that $\gamma(\theta_t)>0$ is equivalent to:
\begin{equation}%\label{e:gamma_cond}
    \lambda(s) > \frac{s^2}{1+s^2}, \quad s\in [0,\infty],
\end{equation}
where $\gamma(\theta_t)$ is as defined in \eqref{eq:ke_polar_form}. That is
\begin{equation}
    \gamma(\theta_t) = 1 - \frac{1-\lambda\parr{\tan(\theta_t)}}{\cos^2(\theta_t)}.
\end{equation}
Thus the condition $\gamma(\theta_t)>0$ holds if and only if
\begin{align}
    \lambda\parr{\tan(\theta_t)} &> 1 - \cos^2(\theta_t)\\
    & = \sin^2(\theta_t)\\
    & = \frac{\sin^2(\theta_t)}{\cos^2(\theta_t)+\sin^2(\theta_t)}\\
    &= \frac{\tan^2(\theta_t)}{1 +\tan^2(\theta_t)}.
\end{align}
Substituting $\tan(\theta_t)= s$, gives the desired inequality, where for $\theta_t \in [0, \frac{\pi}{2}]$, $\tan(\theta_t) \in [0, \infty]$.\\
\end{proof}
\begin{proof}[Proof of $\lambda$ of standard Gaussian.]
Note, for $q(x_1) = \gN\parr{x_1|0,I}$, $\rho_s(x)$ is
\begin{align}
    \rho_s(x) &= \int \gN\parr{x|x_1,s^{-2}I}\gN\parr{x_1|0,I} dx_1\\
    &= \int \gN\parr{x_1|x,s^{-2}I}\gN\parr{x_1|0,I} dx_1\\
    &= \gN\parr{x|0,(s^{-2}+1)I}.
\end{align}
In addition
\begin{align}
    \E_{\rho_s(x_1|x)}x_1 &= \frac{1}{\rho_s(x)}\int x_1 \rho_s(x|x_1)q(x_1)dx_1\\
    &=  \frac{1}{\rho_s(x)}\int x_1 \gN\parr{x|x_1,s^{-2}I}\gN\parr{x_1|0,I}dx_1\\
    &=  \frac{1}{\rho_s(x)}\int (x +s^{-2}\nabla_x)\gN\parr{x|x_1,s^{-2}I}\gN\parr{x_1|0,I}dx_1\\
    &= \frac{1}{\rho_s(x)} (x +s^{-2}\nabla_x)\rho_s(x)\\
    &= \frac{1}{\rho_s(x)} \parr{x +s^{-2}\frac{-x}{s^{-2}+1}}\rho_s(x)\\
    &= x -s^{-2}\frac{x}{s^{-2}+1}\\
    &= \frac{1}{s^{-2}+1}x.
\end{align}
Hence, $\lambda(s)$ is
\begin{align}
    \lambda(s) &= \frac{1}{d}\E_{\rho_s(x)} \norm{ \E_{\rho_s(x_1|x)}x_1}^2\\
    &= \frac{1}{d}\E_{\rho_s(x)} \norm{ \frac{1}{s^{-2}+1}x}^2\\
    &= \parr{\frac{1}{s^{-2}+1}}^2\frac{1}{d}\E_{\rho_s(x)} \norm{x}^2\\
    &= \frac{1}{s^{-2}+1}\\
    &=\frac{s^2}{1+s^2}.
\end{align}
\end{proof}

\newpage
\subsection{Optimization of the Kinetic Energy $\gE$}
\label{a:optimal_solutions}
We derive the solution for the minimization problem of the Kinetic Energy given in \eqref{eq:ke_polar_form}. 
\begin{reptheorem}{thm:main_optimal_path}
    The minimizer of the Kinetic Energy (\eqref{eq:ke_polar_form}) over all conditional probabilities $(r_t,\theta_t)\in \gB$ (\eqref{e:B}) such that $\gamma(\theta_t)>0$ for $t>0$ satisfy 
    \begin{align}
        &r_t = \sqrt{1 -bt  + b t^2},\\%\label{eq:op_r}\\
        &\dot{\theta}_t = \frac{1}{1 -bt + bt^2} \sqrt{\frac{b-\frac{1}{4}b^2}{\gamma(\theta_t)}}%\label{eq:op_theta}
    \end{align}
    where $b\in [0,4]$ is determined by the boundary condition on $\theta_t$.
\end{reptheorem}
\begin{proof}[Proof of theorem \ref{thm:main_optimal_path}.]
    The Kinetic energy parameterized by $r_t$ and $\theta_t$ given in \eqref{eq:ke_polar_form} is
        \begin{equation}
        \gE(r_t,\theta_t) = \int_0^1 \dot{r}_t^2 + r_t^2 \dot{\theta}_t^2\gamma(\theta_t)  dt.
    \end{equation}
    The Euler-Lagrange equations for the Kinetic Energy are
    \begin{align}
    & \frac{\partial\gE_t}{\partial r_t} - \frac{d}{dt}\frac{\partial\gE_t}{\partial \dot{r}_t} =0\\
    & \frac{\partial\gE_t}{\partial \theta_t} - \frac{d}{dt}\frac{\partial\gE_t}{\partial \dot{\theta}_t} =0,
\end{align}
where $\gE_t$ is the integrand of $\gE(r_t,\theta_t)$. Carrying out the derivation gives
\begin{align}
    & r_t\dot{\theta}_t^2\gamma(\theta_t) - \ddot{r}_t = 0 \label{eq:el_r}\\
    & 4r_t\dot{r}_t\dot{\theta}_t\gamma(\theta_t) + 2r_t^2\ddot{\theta}_t\gamma(\theta_t) + r_t^2\dot{\theta}_t^2\frac{d}{d\theta_t}\gamma(\theta_t) = 0.
\end{align}
We notice that the second equation can be separated by dividing by $r_t^2\dot{\theta_t}\gamma(\theta_t)$, that gives
\begin{equation}
    \frac{4\dot{r}_t}{r_t} + 2\frac{\ddot{\theta}_t}{\dot{\theta}_t} + \dot{\theta}_t\frac{\frac{d}{d\theta_t}\gamma(\theta_t)}{\gamma(\theta_t)} = 0
\end{equation}
Notice that all terms can be written as derivative in time, that is 
\begin{align}
    & \frac{4\dot{r}_t}{r_t} = 4\frac{d}{dt}\log(r_t)\\
    & 2\frac{\ddot{\theta}_t}{\dot{\theta}_t} =2\frac{d}{dt}\log(\dot{\theta_t})\\
    & \dot{\theta}_t\frac{\frac{d}{d\theta_t}\gamma(\theta_t)}{\gamma(\theta_t)} =\frac{d}{dt}\log(\gamma(\theta_t)).
\end{align}
Preforming the anti-derivative and afterwards exponentiation of both sides, gives
\begin{equation}\label{eq:el_constant_of_motion}
    r_t^4\dot{\theta}_t^2\gamma(\theta_t) = C,
\end{equation}
where $C>0$ is a constant depending on the boundary condition. Substitute back to the first Euler-Lagrange equation, that is \eqref{eq:el_r}, gives
\begin{equation}
    \ddot{r}_t = \frac{C}{r^3} 
\end{equation}
that is solved by 
\begin{equation}
    r_t = \sqrt{c +bt +at^2},
\end{equation}
where $a,b,c$ satisfy 
\begin{equation}
    ac -\frac{1}{4}b^2 = C >0.
\end{equation}
Finally imposing boundary conditions on $r_t$  as in \ref{lem:polar_bc} gives
\begin{align}
     1 = \sqrt{c},\quad 1 = \sqrt{c +b +a}.
\end{align}
We solve for $c$ and $a$, and get 
\begin{align}
    & r_t = \sqrt{1 -bt +bt^2}\\
    & C = b -\frac{1}{4}b^2,
\end{align}
where the positivity condition on $C$ gives $b \in [0,4]$. To get the equation for $\dot{\theta}_t$ we substitute the solution for $r_t$  and $C$  in \eqref{eq:el_constant_of_motion}, that gives
\begin{equation}
    \dot{\theta}_t = \frac{1}{1 -bt + bt^2} \sqrt{\frac{b-\frac{1}{4}b^2}{\gamma(\theta_t)}}.
\end{equation}
\end{proof}

\subsection{ Boundary condition of polar coordinates }
\begin{lemma}\label{lem:polar_bc}
For $r_t \in [0,\infty],\ \theta_t \in [0, \frac{\pi}{2}]$ the boundary conditions on $r_t$ and $\theta_t$ are
\begin{align}
    & r_0 = 1, & r_1 = 1,\\
    & \theta_0 = 0, & \theta_1 = \frac{\pi}{2}.
\end{align}
\end{lemma}
\begin{proof}[Proof of lemma \ref{lem:polar_bc}.]
    The values of $m_t, a_t$ at $t=0$ and $t=1$ are
\begin{align}
    & m_0 = 1, & m_1 = 0,\\
    & a_0 = 0, & a_1 = 1.
\end{align}
The relation between $r_t,\ \theta_t$ and $a_t,\ m_t$ is defined in \eqref{e:polar_change}, that is
\begin{equation}
    \begin{bmatrix}
    a_t \\ m_t
    \end{bmatrix} = 
    r_t \begin{bmatrix}
    \sin(\theta_t) \\ \cos(\theta_t)
    \end{bmatrix}.
\end{equation}
Substituting $t=0$ and $t=1$ gives
\begin{align}
    & 1 = r_0\cos(\theta_0), & 0 = r_1\cos(\theta_1),\\
    & 0 = r_0\sin(\theta_0), & 1 = r_1\sin(\theta_1).
\end{align}
Since neither $r_0=0$ is a solution nor $r_1=0$, it must be that
\begin{align}
    & \theta_0 = 0, & \theta_1 = \frac{\pi}{2}.
\end{align}
Plugging this back gives the boundary conditions also for $r_0$ and $r_1$.
\end{proof}

\newpage
\subsection{Data separation in high dimensions}
\label{aa:data_separation}
\begin{reptheorem}{thm:ke_squeeze}
Let $q$ be an arbitrary normalized finite data distribution. Assume all $(a_t,m_t)\in\gA$ in \eqref{e:A} satisfy: (i) $\frac{a_t}{m_t}$ is strictly monotonically increasing; and (ii) $a_t,m_t$ have uniform bounded Sobolev $W^{1,\infty}$ norm, \ie,  $\norm{a_t}_{1,\infty},\norm{m_t}_{1,\infty}\leq M$. Then, for all $(a_t,m_t)\in \gA$ 
\begin{equation}
    \gE_c(a_t,m_t) - 6M^2\frac{n}{\sqrt{d}}\le \gE(a_t,m_t)  \le \gE_c(a_t,m_t).
\end{equation}
\end{reptheorem}
\begin{proof}[Proof of theorem \ref{thm:ke_squeeze}]
    The KE given in \eqref{e:ke_explicit} is
    \begin{equation}
         \gE(a_t,m_t) =\gE_c(a_t,m_t) - \int_0^1 \beta_t^2 \brac{1-\lambda\parr{\frac{a_t}{m_t}}}dt.
    \end{equation}
    The fact that $\lambda$ is bounded above by $1$, shown in \eqref{e:lam_inequlity_above}, gives
    \begin{equation}
        \gE(a_t,m_t) \le \gE_c(a_t,m_t).
    \end{equation}
    For the second inequality, remember that $\beta_t$ defined in \eqref{e:def_beta_alpha}, is
    \begin{equation}
        \beta_t = \dot{a}_t - a_t\frac{\dot{m}_t}{m_t}.
    \end{equation}
    Further note that 
    \begin{equation}
        \frac{d}{dt}\parr{\frac{a_t}{m_t}} = \frac{\dot{a}_t}{m_t} - a_t\frac{\dot{m}_t}{m_t^2} = \frac{\beta_t}{m_t}. 
    \end{equation}
$\frac{a_t}{m_t}$ is strictly monotonic increasing, $\frac{d}{dt}\frac{a_t}{m_t}>0$. Hence both $\frac{\beta_t}{m_t} > 0$ and $m_t\beta_t > 0$ for every $t\in(0,1)$. Furthermore, $a_t,m_t$ have uniform bounded Sobolev $W^{1,\infty}([0,1])$ norm by $M>0$ and thus
\begin{equation}
    0 < m_t\beta_t = \dot{a}_tm_t - a_t\dot{m}_t \le 2M^2.
\end{equation}
The gap between the CKE and KE is
\begin{align}
    \gE_c(a_t,m_t) - \gE(a_t,m_t) &= \int_0^1 \beta_t^2 \brac{1-\lambda\parr{\frac{a_t}{m_t}}}dt\\
    & =  \int_0^1 (m_t\beta_t)\parr{\frac{\beta_t}{m_t}} \brac{1-\lambda\parr{\frac{a_t}{m_t}}}dt\\
    & \le 2M^2 \int_0^1 \frac{d}{dt}\parr{\frac{a_t}{m_t}} \brac{1-\lambda\parr{\frac{a_t}{m_t}}}dt.
\end{align}
Since $\frac{a_t}{m_t}$ is strictly monotonic increasing, we can make the change of variable $s=\frac{a_t}{m_t}$ and apply theorem \ref{thm:lambda_bound} that gives
\begin{align}
    2M^2 \int_0^1 \frac{d}{dt}\parr{\frac{a_t}{m_t}} \brac{1-\lambda\parr{\frac{a_t}{m_t}}}dt & = 2M^2 \int_0^{\infty} \parr{1-\lambda\parr{s}}ds\\
    & \le 6M^2\frac{n}{\sqrt{d}}.
\end{align}
\end{proof}

%\newpage
%%%%%%%%%%%%%%%%%%%%%%%%%%%%%%%%%%%%
\begin{reptheorem}{thm:lambda_bound}
Let $q$ be an arbitrary normalized finite data distribution.
Then,
\begin{equation}
    \int_{0}^{\infty}\abs{1-\lambda(s)}ds \le \frac{3n}{\sqrt{d}}.
\end{equation}
\end{reptheorem}
\begin{proof} [Proof of theoren \ref{thm:lambda_bound}.]
    In the proof we will use Bayes' Rule multiple times. For convenience, we state it explicitly for $\rho_s(x_i|x)$:
    \begin{equation}
        \rho_s(x_i|x) = \frac{\rho_s(x|x_i)q(x_i)}{\rho_s(x)}.
    \end{equation}
    To bound $\lambda(s)$ from above, we apply Jensen inequality
    \begin{align}
        \lambda(s) &= \frac{1}{dn}\sum_{j=1}^n\E_{\rho_s(x|x_j)}\norm{\sum_{i=1}^n x_i\rho_s(x_i|x)}^2\\
        &\le \frac{1}{dn}\sum_{j=1}^n\E_{\rho_s(x|x_j)}\sum_{i=1}^n\norm{x_i}^2\rho_s(x_i|x)\\
        &= \frac{1}{dn}\sum_{i=1}^n\norm{x_i}^2\sum_{j=1}^n\E_{\rho_s(x|x_j)}\rho_s(x_i|x)\\
        &= \frac{1}{dn}\sum_{i=1}^n\norm{x_i}^2\sum_{j=1}^n\E_{\rho_s(x|x_j)}\frac{\rho_s(x|x_i)q(x_i)}{\rho_s(x)}\\
        &= \frac{1}{dn}\sum_{i=1}^n\norm{x_i}^2\sum_{j=1}^n\E_{\rho_s(x|x_j)}\frac{1}{n}\frac{\rho_s(x|x_i)}{\rho_s(x)}\\
        &= \frac{1}{dn}\sum_{i=1}^n\norm{x_i}^2\E_{\rho_s(x)}\frac{\rho_s(x|x_i)}{\rho_s(x)}\\
        &= \frac{1}{dn}\sum_{i=1}^n\norm{x_i}^2\E_{\rho_s(x|x_i)}1\\
        &= \frac{1}{dn}\sum_{i=1}^n\norm{x_i}^2\\
        &= 1.\label{e:lam_inequlity_above}
    \end{align}
    To bound $\lambda(s)$ from below, we first rearrange the expression of $\lambda(s)$
\begin{align}
    \lambda(s) &= \frac{1}{dn}\sum_{j=1}^n\E_{\rho_s(x|x_j)}\norm{\sum_{i=1}^n x_i\rho_s(x_i|x)}^2\\
    &= \frac{1}{dn}\sum_{j=1}^n\E_{\rho_s(x|x_j)}\sum_{i=1}^n \sum_{l=1}^n x_i \cdot x_l\rho_s(x_i|x)\rho_s(x_l|x)\\
    &= \frac{1}{d}\E_{\rho_s(x)}\sum_{i=1}^n \sum_{l=1}^n x_i \cdot x_l\rho_s(x_i|x)\rho_s(x_l|x)\\
    &= \frac{1}{d}\sum_{i=1}^n \sum_{l=1}^n x_i \cdot x_l\E_{\rho_s(x)}\rho_s(x_i|x)\rho_s(x_l|x)\\
    &= \frac{1}{d}\sum_{i=1}^n \sum_{l=1}^n x_i \cdot x_l\E_{\rho_s(x)}\frac{\rho_s(x|x_i)q(x_i)}{\rho_s(x)}\rho_s(x_l|x)\\
    &= \frac{1}{dn}\sum_{i=1}^n \sum_{l=1}^n x_i \cdot x_l\E_{\rho_s(x|x_i)}\rho_s(x_l|x)\\
    &= \frac{1}{dn}\sum_{i=1}^n \sum_{l=1}^n x_i \cdot (x_l-x_i+x_i)\E_{\rho_s(x|x_i)}\rho_s(x_l|x)\\ \label{aeq:first_term}
    &= \frac{1}{dn}\sum_{i=1}^n\sum_{l=1}^n \norm{x_i}^2 \E_{\rho_s(x|x_i)}\rho_s(x_l|x) + \frac{1}{dn}\sum_{i=1}^n \sum_{l=1}^n x_i \cdot (x_l-x_i)\E_{\rho_s(x|x_i)}\rho_s(x_l|x)
\end{align}
Remembering that $\rho_s(x_i|x)$ is a discrete probability density on $x_i,\ i\in [n]$, and the normalization as in \eqref{e:normalization_discrete}, the first term in \ref{aeq:first_term} is
\begin{align}
    \frac{1}{dn}\sum_{i=1}^n\sum_{l=1}^n \norm{x_i}^2 \E_{\rho_s(x|x_i)}\rho_s(x_l|x) & = \frac{1}{dn}\sum_{i=1}^n \norm{x_i}^2 \E_{\rho_s(x|x_i)}\sum_{l=1}^n\rho_s(x_l|x)\\
    & = \frac{1}{dn}\sum_{i=1}^n \norm{x_i}^2 \E_{\rho_s(x|x_i)}1\\
    & = \frac{1}{dn}\sum_{i=1}^n \norm{x_i}^2\\
    & =1.
\end{align}
We substitute back to the expression of $\lambda(s)$ and apply Cauchy–Schwarz inequality
\begin{align}
    \lambda(s) &= 1 + \frac{1}{dn}\sum_{i=1}^n \sum_{\substack{l=1 \\ l\neq i}}^n x_i \cdot (x_l-x_i)\E_{\rho_s(x|x_i)}\rho_s(x_l|x)\\
    & \ge 1 - \frac{1}{dn}\sum_{i=1}^n \sum_{\substack{l=1 \\ l\neq i}}^n \norm{x_i} \norm{x_l-x_i}\E_{\rho_s(x|x_i)}\rho_s(x_l|x).\label{e:lam_inequlity_below}
\end{align}
Applying the inequalities in equations \ref{e:lam_inequlity_above}, \ref{e:lam_inequlity_below}  and lemma \ref{lem:key_lem_new} gives
\begin{align}
    \int_0^{\infty} \abs{1-\lambda(s)}ds &= \int_0^{\infty} \parr{1-\lambda\parr{s}}ds\\
    & \le \int_0^{\infty}\frac{1}{dn}\sum_{i=1}^n \sum_{\substack{l=1 \\ l\neq i}}^n \norm{x_i} \norm{x_l-x_i}\E_{\rho_s(x|x_i)}\rho_s(x_l|x)ds\\
    & \le \frac{1}{dn}\sum_{i=1}^n \sum_{\substack{l=1 \\ l\neq i}}^n \norm{x_i}\int_0^{\infty} \norm{x_l-x_i}\eta\parr{s\norm{x_l-x_i}}ds\\
    & = \frac{1}{dn}\sum_{i=1}^n \sum_{\substack{l=1 \\ l\neq i}}^n \norm{x_i}\int_0^{\infty}\eta(t)dt\\
    & \le \frac{3}{dn}\sum_{i=1}^n \sum_{\substack{l=1 \\ l\neq i}}^n \norm{x_i}\\
    & \le \frac{3n}{\sqrt{d}},
\end{align}
where in the last equality we use Jensen inequality and \eqref{e:q_var} to bound the sum of norms, $\norm{x_i}$, as shown below
\begin{align}
    1 = \frac{1}{dn}\sum_{i=1}^n\norm{x_i}^2 \ge \frac{1}{d}\parr{\frac{1}{n}\sum_{i=1}^n\norm{x_i}}^2.
\end{align}
\end{proof}
%%%%%%%%%%%
%%%%%%%%%%%
\newpage
\begin{replemma}{lem:key_lem_new}
    Let $q$ be an arbitrary normalized finite data distribution. Then for every $s>0$,
    \begin{align}
           \E_{\rho_s(x|x_j)}\rho_s(x_i|x) \le \eta\parr{s\norm{x_i-x_j}},
    \end{align}
    where $\eta(t)$ is integrable in $[0,\infty)$ and  
    \begin{equation}
        \int_0^{\infty}\eta(t) \le 3.
    \end{equation}
\end{replemma}
\begin{proof}[Proof of lemma \ref{lem:key_lem_new}]
 Remembering that $p(x)=\gN(0,I)$,
\begin{align}
    \E_{\rho_s(x|x_j)}\rho_s(x_i|x) &= \E_{\rho_s(x|x_j)}\frac{e^{-\frac{s^2}{2}\norm{x - x_i}^2}}{\sum_{l=1}^ne^{-\frac{s^2}{2}\norm{x - x_l}^2}}\\
    &= \E_{p(x)}\frac{e^{-\frac{s^2}{2}\norm{xs^{-1} +x_j - x_i}^2}}{\sum_{l=1}^ne^{-\frac{s^2}{2}\norm{xs^{-1} +x_j - x_l}^2}}\\
    & \le \E_{p(x)}\frac{e^{-\frac{s^2}{2}\norm{xs^{-1} +x_j - x_i}^2}}{e^{-\frac{s^2}{2}\norm{xs^{-1} +x_j - x_i}^2} +e^{-\frac{s^2}{2}\norm{xs^{-1}}^2}}\\
    & = \E_{p(x)}\frac{1}{1 +e^{\frac{s^2}{2}\norm{x_i - x_j}^2-s\parr{x_i - x_j}\cdot x}}
\end{align}
Since $x \sim p(x) = \gN\parr{0,I}$ is $O(d)$ invariant, the last term depends only on the norm of $x_i - x_j$. That is, we denote by $z \in \R$ the component of $x$ in the direction of $x_i -x_j$
\begin{equation}
    z = \frac{x_i -x_j}{\norm{x_i -x_j}} \cdot x
\end{equation}
So $z \sim \gN(0, 1)$ and 
\begin{align}
    \E_{p(x)}\frac{1}{1 +e^{\frac{s^2}{2}\norm{x_i - x_j}^2-s\parr{x_i - x_j}\cdot x}} &= E_{\gN(z|0,1)}\frac{1}{1 + e^{\frac{s^2}{2}\norm{x_i -x_j}^2 -s\norm{x_i - x_j}z} }\\
    \\&= E_{\gN(z|0,1)}\frac{1}{1 + e^{\frac{1}{2}\parr{s\norm{x_i -x_j}}^2 -\parr{s\norm{x_i - x_j}}z} }
\end{align}
And we define $\eta(t)$ to be 
\begin{equation}
    \eta(t) = E_{\gN(z|0,1)}\frac{1}{1 + e^{\frac{t^2}{2} -tz} }.
\end{equation}
So it is left to show that $\eta(t)$ is integrable in $[0,\infty)$ and bound its integral. For every $t>0$
\begin{align}
    \eta(t) &= E_{\gN(z|0,1)}\frac{1}{1 + e^{\frac{t^2}{2} -tz} }\\
    &= \frac{1}{\sqrt{2\pi}}\int_{-\infty}^{\infty}\frac{e^{-\frac{z^2}{2}}dz}{1 + e^{\frac{t^2}{2} -tz} }\\
    &= \frac{1}{\sqrt{2\pi}}\int_{-\infty}^{\infty}\frac{dz}{e^{\frac{z^2}{2}} + e^{\frac{1}{2}(z-t)^2}}\\
    &= \frac{1}{\sqrt{2\pi}}\int_{-\infty}^{\infty}\frac{dz}{e^{\frac{1}{2}\parr{z+\frac{t}{2}}^2} + e^{\frac{1}{2}\parr{z-\frac{t}{2}}^2}}\\
    &= \sqrt{\frac{2}{\pi}}\int_{0}^{\infty}\frac{dz}{e^{\frac{1}{2}\parr{z+\frac{t}{2}}^2} + e^{\frac{1}{2}\parr{z-\frac{t}{2}}^2}}\\
    & \le \sqrt{\frac{2}{\pi}}\int_{0}^{\infty}e^{-\frac{1}{2}\parr{z+\frac{t}{2}}^2}dz\\
    &= \sqrt{\frac{2}{\pi}}\int_{\frac{t}{2}}^{\infty}e^{-\frac{z^2}{2}}dz,
\end{align}
where in the third to last equality we used the fact that the integrand is symmetric w.r.t $z\too-z$. Now on the one hand, for every $t>0$
\begin{align}
    \sqrt{\frac{2}{\pi}}\int_{\frac{t}{2}}^{\infty}e^{-\frac{z^2}{2}}dz & \le \sqrt{\frac{2}{\pi}}\int_{0}^{\infty}e^{-\frac{z^2}{2}}dz = 1\label{e:eta_bound1}
\end{align}
On the other hand, $\frac{2z}{t}>1$ for every $z>\frac{t}{2}$. Hence
\begin{align}
    \sqrt{\frac{2}{\pi}}\int_{\frac{t}{2}}^{\infty}e^{-\frac{z^2}{2}}dz & \le \sqrt{\frac{2}{\pi}}\int_{\frac{t}{2}}^{\infty}\frac{2z}{t}e^{-\frac{z^2}{2}}dz\\
    & = \sqrt{\frac{2}{\pi}}\int_{\frac{t^2}{4}}^{\infty}\frac{1}{t}e^{-\frac{z'}{2}}dz'\\
    & = \sqrt{\frac{2}{\pi}}\frac{1}{t}e^{-\frac{t^2}{8}}\label{e:eta_bound2}
\end{align}
Combining the two equations \ref{e:eta_bound1} and \ref{e:eta_bound2}, we get that for every $t>0$
\begin{equation}
    \eta(t) \le \min\set{1, \sqrt{\frac{2}{\pi}}\frac{1}{t}e^{-\frac{t^2}{8}}},
\end{equation}
which is integrable in $[0,\infty)$. Furthermore,
\begin{align}
    \int_0^{\infty} \eta(t)dt &= \int_0^{1} \eta(t)dt +\int_1^{\infty} \eta(t)dt\\
    &\le \int_0^{1}dt +\int_1^{\infty} \sqrt{\frac{2}{\pi}}\frac{1}{t}e^{-\frac{t^2}{8}}dt\\
    & \le 1 +\int_0^{\infty} \sqrt{\frac{2}{\pi}}e^{-\frac{t^2}{8}}dt\\
    &=3
\end{align}
\end{proof}

\newpage
\section{Further implementation details}
\label{a:implementation-details}
We model $\theta_t$ using the following architecture 
\begin{equation}
    \theta_t =\frac{\pi}{2}\frac{\abs{m(t)- m(0)}}{\abs{m(1)- m(0)}}
\end{equation}
where $m$ is an MLP defined as 
\begin{equation}
    \phi(t) = \texttt{sigmoid}\Big(L_1(t) +L_3\parr{2 \texttt{sigmoid}\parr{L_2(t)} -1}\Big)
\end{equation}
and $L_i$, $i\in [3]$, are linear layers: $L_1:\Real\too \Real$; $L_2:\Real\too\Real^2$; and $L_3:\Real^2\too \Real$. In total the model for $\theta_t$ uses $10$ learnable parameters. We optimize the KE using ADAM optimizer and Cosine Annealing scheduler. 

For the CNF training in Section \ref{ss:FM_with_KO} we model $v_t(x)$ as follows: (i) For the 2D data it is an MLP with $5$ layers consisting of $512$ neurons in each layer. (ii) For the image datasets (CIFAR, ImageNet-32), we used U-Net architecture as in \citep{dhariwal2021diffusion}, where the particular architecture hyper-parameters are detailed in Table \ref{tab:hyper-params}. In this table we also provide training hyper-parameters. All baselines (IS, DDPM) are implemented in exactly the same setting of architecture and training hyper-parameters. 

\begin{table}
\centering
\section{Additional tables}
\begin{tabular}{l c c }
\toprule
 & CIFAR10  & ImageNet-32  \\
\midrule
Channels & 256 & 256   \\
Depth & 2  & 3  \\
Channels multiple & 1,2,2,2 & 1,2,2,2  \\
Heads & 4 & 4  \\
Heads Channels & 64 & 64  \\
Attention resolution & 16 & 16,8 \\
Dropout & 0.1 & 0.0  \\
Effective Batch size & 256 & 1024  \\
GPUs & 2 & 4  \\
Epochs & 2000 & 400  \\
Iterations & 392k & 500k   \\
Learning Rate & 5e-4 & 1e-4  \\
Learning Rate Scheduler & Polynomial Decay & Polynomial Decay\\
Warmup Steps & 19.6k & 20k  \\
\bottomrule
\end{tabular}
%}
\caption{Hyper-parameters used for training. }
\label{tab:hyper-params}
\end{table}

\begin{table}\centering
%\ra{1.2}
\setlength{\tabcolsep}{2.5pt}
%\resizebox{\linewidth}{!}{%
% \small
\begin{tabular}{@{} l r r r r r r r r  @{}}\toprule
  & \multicolumn{3}{c}{\bf CIFAR-10} &
  & \multicolumn{3}{c}{\bf ImageNet 32$\times$32}  \\
\cmidrule(lr){2-4} \cmidrule(lr){5-8} 
Model & {$K$=1} & {$K$=20} & {$K$=50}
& & {$K$=1} & {$K$=5} & {$K$=15} \\
\cmidrule(r){1-1} \cmidrule(lr){2-4} \cmidrule(lr){5-8} 
\textit{\small Ablation}\\
\;\; DDPM~{\tiny\citep{ho2020denoising}} &
3.23 & 3.13 & 3.11 & & 
3.69 & 3.64 & 3.61 \\
\;\; SI~{\tiny\citep{albergo2022si}}  & 
3.13 & 3.02 & 3.00 & & 
3.67 & 3.61 & 3.58 \\

\cmidrule(r){1-1} \cmidrule(lr){2-4} \cmidrule(lr){5-8}
\;\; FM \textsuperscript{w}/ OT~{\tiny\citep{lipman2022flow}} & 
3.10 & 3.00 & 2.98 & & 
3.69 & 3.64 & 3.60 \\
\;\; FM \textsuperscript{w}/ KO & 
3.10 & 3.00 & 2.97 & & 
3.66 & 3.60 & 3.57 \\

\bottomrule
\end{tabular}
%}
\caption{BPD results for different uniform dequantization samples $k$, following the protocol of \cite{lipman2022flow}.}
\label{tab:nll_k_results}
\end{table}
%%%%%%%%%%%%%%%%%%%%%%%%%%%%%%%%%%%%%%%%%%%%%%%%%%%%%%%%%%%%%%%%%%%%%%%%%%%%%%%
%%%%%%%%%%%%%%%%%%%%%%%%%%%%%%%%%%%%%%%%%%%%%%%%%%%%%%%%%%%%%%%%%%%%%%%%%%%%%%%
\section{Runtime Details} \label{a:runtime}
First, during deployment (model sampling) there is no change to the original algorithm of sampling from CNFs. Second, during training we use the optimal $a_t,m_t$ found by optimizing \eqref{eq:ke_polar_form} and \eqref{eq:lambda_estimator}. The solution $a_t,m_t$ is modeled with a 
-parameter MLP and causes no noticeable change in iteration time compared to Cond-OT. Third, the optimization of $a_t,m_t$ is done once as a preprocessing step and takes less than a minute. Fourth, the most time consuming part of our algorithm is approximating the data separation function $\lambda$ but it is done once per dataset, so can be reused across training sessions as long as the dataset hasn’t changed. The complexity of approximating $\lambda$ can be deduced from \eqref{eq:lambda_estimator} and it is $O(kn^2)$, where $k$ is the number of samples from $p=\gN\parr{0, I}$ and $n$ is the number of samples from $q$. In Table \ref{tab:runtime_lambda} we present the runtime of computing $\lambda$ for $k=1$ and $n=50,000$ on 1 GPU (Quadro RTX 8000) for a number of image datasets of different image sizes.
\begin{table}
\centering
\begin{tabular}{l c }
\toprule
 Dataset & Time [s] \\
\midrule
ImageNet 8 & 30.4   \\
ImageNet 16 & 42.0   \\
ImageNet 32/CIFAR10 & 84.1   \\
ImageNet 64 & 328.6   \\
\bottomrule
\end{tabular}
%}
\caption{ Runtime of computing $\lambda$ for $k=1$ and $n=50,000$ on 1 GPU (Quadro RTX 8000).}
\label{tab:runtime_lambda}
\end{table}

\begin{figure}
    \centering
\section{Additional Figures}\label{a:additional_figures}    \includegraphics[width=\textwidth,height=0.95\textheight,keepaspectratio]{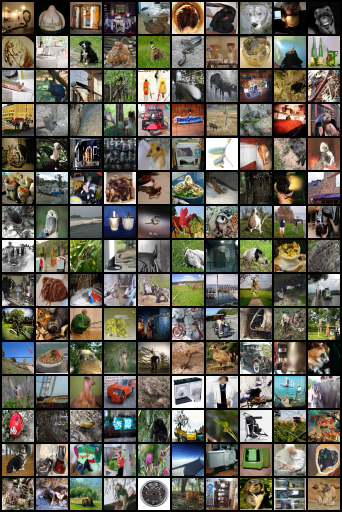}
    \caption{Non-curated unconditional \textit{face-blurred} ImageNet-32 generated images of a CNF model trained with FM-KO. }
    \label{fig:imagenet32_samples}
\end{figure}

% Figure of trajectories 
\begin{figure*}[h!]
    \centering
    \begin{tabular}{@{\hspace{0pt}}c@{\hspace{8pt}}c@{\hspace{8pt}}c@{\hspace{0pt}}}
         \includegraphics[width=0.30\textwidth]{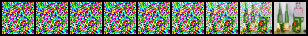} & \includegraphics[width=0.30\textwidth]{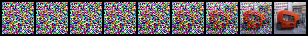} &
         \includegraphics[width=0.30\textwidth]{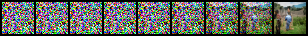} \vspace{-3pt} \\
         & \scriptsize{DDPM} & \vspace{1pt} \\
         \includegraphics[width=0.30\textwidth]{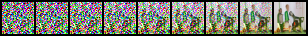} & \includegraphics[width=0.30\textwidth]{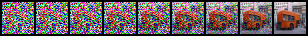} &
         \includegraphics[width=0.30\textwidth]{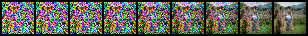} \vspace{-2pt} \\
         & \scriptsize{Stochastic Interpolants} & \vspace{1pt}\\
         \includegraphics[width=0.30\textwidth]{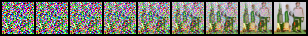} & \includegraphics[width=0.30\textwidth]{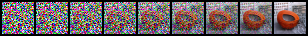} &
         \includegraphics[width=0.30\textwidth]{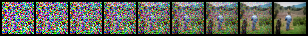} \vspace{-2pt}\\
         & \scriptsize{Flow Matching \textsuperscript{w}/ Cond-OT } & \vspace{1pt} \\
         \includegraphics[width=0.30\textwidth]{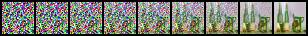} & \includegraphics[width=0.30\textwidth]{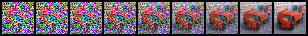} &
         \includegraphics[width=0.30\textwidth]{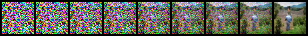} \vspace{-2pt} \\
         & \scriptsize{Flow Matching \textsuperscript{w}/ KO } & \vspace{1pt} \\
         \vspace{-10pt}
    \end{tabular}
    \caption{Sample trajectories of a CNF models trained with DDPM, SI, Cond-OT and kinetic optimal paths, on \textit{face-blurred} ImageNet 32 data.}
    \label{fig:trajectories_imagenet32}
\end{figure*}
\end{document}